\documentclass{article}

\usepackage{microtype}
\usepackage{graphicx}
\usepackage{subcaption}
\usepackage{booktabs}
\usepackage{float}
\usepackage{hyperref}
\usepackage{amsmath}
\usepackage{amssymb}
\usepackage{mathtools}
\usepackage{amsthm}
\usepackage{multirow}
\usepackage[capitalize,noabbrev]{cleveref}
\usepackage[textsize=tiny]{todonotes}

\usepackage[accepted]{icml2026}

\theoremstyle{plain}
\newtheorem{theorem}{Theorem}[section]

\theoremstyle{definition}

\theoremstyle{remark}

\icmltitlerunning{Training-Free Bayesian Filtering with Generative Emulators}

\begin{document}

\twocolumn[
  
    \icmltitle{Training-Free Bayesian Filtering with Generative Emulators}

    \begin{icmlauthorlist}
        \icmlauthor{Thomas Savary}{ulg}
        \icmlauthor{François Rozet}{ulg}
        \icmlauthor{Gilles Louppe}{ulg}
    \end{icmlauthorlist}

    \icmlaffiliation{ulg}{SAIL, Montefiore institute, University of Liège, Belgium}

    \icmlcorrespondingauthor{Thomas Savary}{tsavary@uliege.be}

    \icmlkeywords{Generative Models, Generative Emulators, Filtering, Data Assimilation, Diffusion, Stochastic Interpolants, ICML}

  \vskip 0.3in
]

\printAffiliationsAndNotice{}

\begin{abstract}
   Bayesian filtering is a well-known problem that aims to estimate plausible states of a dynamical system from observations. Among existing approaches to solve this problem, particle filters are theoretically exact for non-linear dynamics and observations, but suffer from poor scalability in high dimensions. In this work, we show that diffusion-based emulators of dynamical systems can be used to implement, without additional training, an optimal variant of particle filters that has remained largely unexplored due to implementation challenges with classical numerical solvers. Experiments on nonlinear chaotic systems, including atmospheric dynamics, demonstrate that the proposed approach successfully scales particle filtering to high-dimensional settings.
\end{abstract}

\section{Introduction}

    Numerical simulation of dynamical systems is a central tool in science, enabling the exploration and prediction of complex phenomena that are not accessible to direct experimentation. Traditionally, it relies on modeling the dynamics using partial differential equations, which are then solved with numerical methods \cite{Lorenz63, NavierStokesSimulation, Hairer}. This approach requires significant modeling and numerical effort, and becomes computationally expensive for large dynamical systems. 
    
    Recently, deep neural networks  have emerged as a compelling alternative, achieving competitive accuracy at substantially lower computational cost \cite{cnn_simulation, GraphCast}. In particular, generative models \cite{diffusion, FlowMatching, Interpolants} attract growing interest due to their ability to capture high-dimensional, multimodal distributions, making them promising candidates for efficient simulation of dynamical systems \cite{probabilistic_forecasting_interpolants, lola}.

    A frequently overlooked aspect of these methods is their dependence on an initial condition to start the simulation. The latter is particularly important for chaotic systems, where small deviations in the initial state grow exponentially over time \cite{intro_to_chaos}. Since, in most real-world settings, it is difficult to accurately estimate the state of the system at a given time, a common practice, known as data assimilation, is to use a set of observations to infer the most probable states \cite{data_assimilation}. Various algorithms have been proposed to tackle this problem \cite{lorenc_data_assimilation, EnKF_Evensen}, and this paper contributes to ongoing efforts to adapt and design data assimilation algorithms with generative models \cite{SDA, DiffDA}.

    \subsection{Problem statement}
        We consider a discrete-time Markovian dynamical system with unknown state $x^{k}$ at time step $k$. The goal, known as Bayesian filtering \cite{BayesianFiltering}, is to estimate $x^{k}$ from past and current observations $y^{1:k} = (y^{1}, y^{2}, \dots, y^{k})$, that is, to approximate the posterior distribution $p(x^{k} \mid y^{1:k})$. For example, in weather forecasting, $x^{k}$ represents the state of the atmosphere at a given time, while observations $y^{1:k}$ come from ground weather stations and satellites \cite{Arome3DEnVAR}. We further assume access to a generative emulator of the dynamics, in the form of a diffusion model (see Section \ref{Subsection:diffusion}), which allows generating probable future states $x^{k+1}$ from a given current state $x^{k}$, that is, drawing samples from $p(x^{k+1} \mid x^{k})$.

    \subsection{Contributions}
        In this work, we propose a simple but effective method that adapts diffusion emulators to perform Bayesian filtering without additional training. The method is mathematically grounded and establishes an elegant connection between generative models and particle filters \cite{PF}, in particular the fully adapted particle filter \cite{FA_APF}. We show that, contrary to common belief \cite{HighDimensional_PF}, particle filters can be applied in high-dimensional problems when combined with generative models, even with relatively few particles. To illustrate this, we apply our method on GenCast, a global diffusion-based emulator of the atmosphere \cite{GenCast}. The code is available at \href{https://github.com/ThomasSavary08/FA-APF}{https://github.com/ThomasSavary08/FA-APF}.
    
\section{Preliminaries}

    \subsection{Data assimilation and Bayesian filtering}

        As explained earlier, the goal of data assimilation algorithms is to estimate the state $x^{k}$ of a dynamical system at time $k$ from a set $\mathcal{I}$ of observations $\{y^{i}\}_{i \in \mathcal{I}}$, a model of the system dynamics, and a model for the observations. Formally, following the notations of \citet{data_assimilation}, we define the dynamics as 
        \begin{equation} \label{eq:dynamic_model}
            x^{k+1} = \mathcal{M}(x^{k}, \lambda) + \eta^{k+1} \sim p(x^{k+1} \mid x^{k}),
        \end{equation}
        where $\mathcal{M}$ is a transition model, $\lambda$ the parameters of the model, and $\eta^{k+1}$ a stochastic additive term intended to represent the error between the model prediction and the true unknown process initialized from the perfect initial condition. This term accounts for the cumulative effect of errors in the parameters $\lambda$, errors in the transition model $\mathcal{M}$, and the effect of unresolved scales \cite{data_assimilation}. In Section \ref{section:Method}, we assume access to a diffusion model that directly samples from $p(x^{k+1} \mid x^{k})$.
        
        For the observations, we assume that the measurement process can be modeled by an operator $\mathcal{H}: \mathbb{R}^{n} \longrightarrow \mathbb{R}^{d}$ mapping the state space to the observation space, leading to the following formulation
        \begin{equation} \label{eq:observation_model}
            y^{k} = \mathcal{H}(x^{k}) + \varepsilon^{k}.
        \end{equation}
        Similarly to the model error, the additive stochastic term $\varepsilon^{k}$ accounts for the instrumental error of observing devices (such as satellites), and deficiencies in the formulation of the observation operator itself \cite{data_assimilation}. As commonly assumed in other data assimilation works \cite{SDA, EnSF, Appa}, we model this additive term as a zero-mean Gaussian with covariance $\Sigma_{y}$, leading to a Gaussian distribution for the observations
        \begin{equation} \label{eq:observation_distribution}
            y^{k} \sim \mathcal{N}(\mathcal{H}(x^{k}), \Sigma_{y}).
        \end{equation}

        Depending on the set of observations used to estimate the system state, we commonly distinguish two main problems in data assimilation. When past, current, and future observations are considered, the problem is referred to as  Bayesian smoothing \cite{BayesianFiltering}. Solutions to this problem are primarily used to construct reanalysis datasets such as ERA5 \cite{ERA5}. When only past and current observations are used, the problem reduces to Bayesian filtering, as introduced in the problem statement. Solutions to this problem are mainly used operationally to determine initial conditions for simulations \cite{Arome3DEnVAR}.

    \subsection{Particle filters} \label{subsection:PF}

        Particle filters are a family of data assimilation algorithms designed to estimate the filtering distribution $p(x^{k} \mid y^{1:k})$ by a discrete probability measure $\mu^{k} = \sum_{i=1}^{N} w^{k}_{i} \delta_{x^{k}_{i}}$ that converges weakly
        \begin{equation} \label{eq:weak_convergence_pf}
            \sum_{i = 1}^{N} w^{k}_{i}g(x^{k}_{i}) \underset{N \to +\infty}{\longrightarrow} \int g(x^{k})p(x^{k} \mid y^{1:k}) \mathrm{d}x^{k},
        \end{equation}
        where $x^{k}_{i}$ are the particles at time $k$, $w^{k}_{i}$ the associated weights, $y^{1:k}$ the observations and $g$ any continuous and bounded function \cite{PF}.

        Particle filters are derived from the Bayesian filtering recursion: at each time step $k$, the next filtering distribution is obtained through a prediction and update step
        \begin{align} \label{eq:pf_equations}
            &p(x^{k+1} \mid y^{1:k}) = \int p(x^{k+1} \mid x^{k}) p(x^{k} \mid y^{1:k}) \mathrm{d}x^{k}, \\
            &p(x^{k+1} \mid y^{1:k+1}) \propto p(y^{k+1} \mid x^{k+1}) p(x^{k+1} \mid y^{1:k}).
        \end{align}
        In practice, the prediction step is done by sampling the next particles from a proposal distribution $q(x^{k+1} \mid x^{k}, y^{k+1})$ that represents any transition distribution conditioned on the current state $x^{k}$ and, optionally, the next observation $y^{k+1}$. Weights are then updated to correct the mismatch between the predicted and true posterior distribution 
        \begin{equation} \label{eq:unnormalized_weights}
            \hat{w}^{k+1}_{i} = \frac{p(y^{k+1} \mid x^{k+1}_{i})p(x^{k+1}_{i} \mid x^{k}_{i})}{q(x^{k+1}_{i} \mid x^{k}_{i}, y^{k+1})} \times w_{i}^{k},
        \end{equation}
        and subsequently normalized to obtain a valid probability measure
        \begin{equation} \label{eq:normalized_weights}
            w^{k+1}_{i} = \frac{\hat{w}_{i}^{k}}{\sum_{j=1}^{N} \hat{w}_{j}^{k}}.
        \end{equation}

        Unlike classical approaches such as ensemble Kalman filters \cite{EnKF_Evensen} or variational methods \cite{4D_VAR}, particle filters offer the theoretical advantage of handling non-linear transition and observation models, which arise in most dynamical systems of interest. However, in high-dimensional systems, particle filters suffer from the curse of dimensionality \cite{pf_degeneracy}. This phenomenon, known as degeneracy, corresponds to the
        situation where only a small subset of particles have non-negligible weights. It is due to the dimension of the observation space: the higher this dimension, the more peaked the likelihood is, and the more unlikely it is for the majority of particles to end up close to the observation \cite{pf_degeneracy}. 
        
        To alleviate this issue, particle filtering algorithms commonly introduce a resampling step, which preserves particle diversity and resets the weights. Also, it has been shown in the literature that using the optimal proposal $q(x^{k+1} \mid x^{k}, y^{k+1}) = p(x^{k+1} \mid x^{k}, y^{k+1}) $ during the sampling step minimizes the variance of the weights and, consequently, reduces degeneracy \cite{OptimalProposal}. However, the use of the optimal proposal is rarely feasible in practice, as it is difficult to implement for standard simulators.

    \subsection{Diffusion models for probabilistic forecasting} \label{Subsection:diffusion}

        Diffusion models \cite{DDPM, diffusion}, are a class of generative models to sample plausible data from a distribution $p(x)$ of interest. Formally, adapting the formulation of \citet{diffusion}, samples $x \in \mathbb{R}^{n}$ from $p(x)$ are progressively perturbed through a diffusion process expressed as a stochastic differential equation (SDE)
        \begin{equation} \label{eq:forward_diffusion_process}
            \mathrm{d}x_{t} = f_{t}x_{t} \mathrm{d}t + g_{t} \mathrm{d}w_{t},
        \end{equation}
        where $f_{t} \in \mathbb{R}$ is the drift coefficient, $g_{t} \in \mathbb{R}_{+}$ the diffusion coefficient, $w_{t} \in \mathbb{R}^{n}$ a standard Wiener process, and $x_{t} \in \mathbb{R}^{n}$ the perturbed sample at time $t \in [0,1]$. Because the SDE is linear, the perturbation kernel from $x$ to $x_{t}$ is Gaussian and takes the form
        \begin{equation} \label{eq:transition_kernel_diffusion}
            p_{t}(x_{t} \mid x) = \mathcal{N}(x_{t} \mid \alpha_{t}x, \Sigma_{t}),
        \end{equation}
        where $\alpha_{t}$ and $\Sigma_{t} = \sigma_{t}^{2}I$ are derived from $f_{t}$ and $g_{t}$ \cite{SDE}. Crucially, the forward SDE \eqref{eq:forward_diffusion_process} has an associated family of reverse SDEs \cite{SDE}
        \begin{equation} \label{eq:reverse_diffusion_equation}
            \mathrm{d}x_{t} = \left [ f_{t}x_{t} - \frac{1 + \eta^{2}}{2}g_{t}^{2} \nabla_{\!x_{t}} \log p_{t}(x_{t}) \right] \mathrm{d}t + \eta g_{t} \mathrm{d}w_{t},
        \end{equation}
        where $\eta$ is a parameter controlling stochasticity. In simpler words, we can draw noise samples from $p(x_{1}) \approx \mathcal{N}(0, \Sigma_{1})$ and gradually remove the noise to obtain samples from $p(x)$ by simulating Equation~\eqref{eq:reverse_diffusion_equation} from $t=1$ to $t=0$ using an appropriate discretization scheme \cite{Exponential_solver, DDIM}.

        In practice, the score function $\nabla_{\!x_{t}} \log p_{t} (x_{t})$ in Equation~\eqref{eq:reverse_diffusion_equation} is unknown, but can be approximated by a neural network $d_{\theta}$ called denoiser and trained with the following loss function
        \begin{equation} \label{eq:diffusion_training}
            \mathcal{L}(\theta) = \mathbb{E}_{x,t,x_{t}} \left [ \lambda_{t} \lVert x - d_{\theta}(t, x_{t}) \rVert_{2}^{2} \right].
        \end{equation}
        In fact, the optimal denoiser is the mean $\mathbb{E}[x \mid x_{t}]$ of $p(x \mid x_{t})$, and is linked to score function through the first order Tweedie's formula (see Appendix \ref{appendix:Tweedie})
        \begin{equation} \label{eq:first_order_Tweedie}
            \nabla_{\!x_{t}} \log p_{t} (x_{t}) = \Sigma_{t}^{-1}(\alpha_{t} \mathbb{E}[x \mid x_{t}] - x_{t}).
        \end{equation}
        
        Given a dataset $\{(x^{k}_{i}, x^{k+1}_{i})\}_{i \in \mathcal{I}}$ of successive states of a dynamical system, we can adapt this paradigm to construct a generative emulator of the dynamic. To do so, we reformulate the training objective given in Equation~\eqref{eq:diffusion_training} as
        \begin{equation} \label{eq:training_conditional_diffusion}
            \mathbb{E}_{(x^{k},x^{k+1}), t, x^{k+1}_{t}} \left [ \lambda_{t} \lVert x^{k+1} - d_{\theta}(t, x^{k}, x^{k+1}_{t}) \rVert_{2}^{2} \right].
        \end{equation}
        In this case, the optimal denoiser is the conditional mean $\mathbb{E}[x^{k+1} \mid x^{k+1}_{t}, x^{k}]$ and is linked to the conditional score function $\nabla_{\!x^{k+1}_{t}} \log p_{t} (x^{k+1}_{t} \mid x^{k})$ through Equation~\eqref{eq:first_order_Tweedie}. Using this score in Equation~ \eqref{eq:reverse_diffusion_equation} yields samples from $p(x^{k+1} \mid x^{k})$, thereby enabling autoregressive emulation of the system dynamics. In the remainder of this work, we assume access to a diffusion-based emulator for a given dynamical system.

\section{Method} \label{section:Method}
    As outlined in the introduction, the objective is to estimate the Bayesian filtering distribution $p(x^{k} \mid y^{1:k})$ using a diffusion-based model trained for sampling the transition distribution $p(x^{k+1} \mid x^{k})$ of a given dynamical system. We show that such generative emulators can be directly used, without additional training, to implement the fully adapted auxiliary particle filter \cite{FA_APF}.
    \subsection{The Fully-Adapted Auxiliary Particle Filter}

        The fully adapted auxiliary particle filter (FA-APF) is a particle filtering algorithm designed to approximate the Bayesian filtering distribution \cite{FA_APF}. It is summarized in Algorithm~\ref{algo:FA-APF}, where $p(x^{0})$ denotes the prior distribution of the initial state, $N$ the number of particles, and $K$ the total number of time steps. The parameters $\alpha$, $\mathrm{N}_{\text{thr}}^{\text{min}}$, $\mathrm{N}_{\text{thr}}^{\text{max}}$ are introduced below.

        \begin{algorithm}    
        \caption{Fully-Adapted Auxiliary Particle Filter}
        \begin{algorithmic}[1]
            \STATE {\bfseries Inputs:} $p(x^{0})$, $N$, $N_{\text{thr}}^{\text{min}, \text{max}}$, $\alpha$, $K$
            \STATE $x^{0}_{i} \sim p(x^{0})$
            \STATE $w^{0}_{i} \gets 1/N$
            \FOR{$k=0$ {\bfseries to} $K-1$}
                \STATE $\mu^{k+1}_{i} \gets \mathbb{E}[x^{k+1} \mid x^{k}_{i}]$
                \WHILE{$N_{\text{eff}}$ not in $[N_{\text{thr}}^{\text{min}}, N_{\text{thr}}^{\text{max}}]$}
                    \STATE update/initialize $\alpha$
                    \STATE $\hat{w}^{k+1}_{i} \gets \left[ p(y^{k+1} \mid \mu^{k+1}_{i})\right]^{\alpha}$
                    \STATE $w^{k+1}_{i} \gets \hat{w}^{k+1}_{i} / \sum_{j=1}^{N} \hat{w}^{k+1}_{j}$
                    \STATE $N_{\text{eff}} \gets 1 / \sum_{i=1}^{N}(w^{k+1}_{i})^{2}$
                \ENDWHILE
                \STATE $a^{k+1}_{i} \sim \text{Cat}(\{w^{k+1}_{i}\}_{1 \leq i \leq N})$
                \vspace{0.2em}
                \STATE $x^{k+1}_{i} \sim p(x^{k+1} \mid x^{k}_{a^{k+1}_{i}}, y^{k+1})$
            \ENDFOR
            \STATE {\bfseries Return} $\mu^{k}_{x} = \frac{1}{N}\sum_{i = 1}^{N}\delta_{x^{k}_{i}}$ for all $k \in [1,K]$
        \end{algorithmic}
        \label{algo:FA-APF}
        \end{algorithm}

        Unlike most particle filters used in practice, FA-APF first computes the particle weights (lines 5–10 of Algorithm~\ref{algo:FA-APF}), and then propagates the particles to the next time step using these weights together with the optimal proposal $q(x^{k+1} \mid x^{k}, y^{k+1}) = p(x^{k+1} \mid x^{k}, y^{k+1})$ (lines 12-13). This choice minimizes the variance of the weights \cite{OptimalProposal} and therefore reduces particle degeneracy.

        To further control degeneracy, following a strategy commonly used in data assimilation methods \cite{LETKF}, we introduce an inflation coefficient $\alpha$ and a control interval $[\mathrm{N}_{\text{thr}}^{\text{min}}, \mathrm{N}_{\text{thr}}^{\text{max}}]$ on the effective sample size, defined as the number of particles with non-negligible weights among the $N$ particles. While this introduces a bias in the approximation of the filtering distribution, it ensures that the variance of particles is non-zero at each time step $k$.

    \subsection{Sampling from the optimal proposal} \label{subsection: Optimal_proposal}
    To apply the FA-APF, we must sample from the optimal proposal distribution $p(x^{k+1} \mid x^{k}, y^{k+1})$, which is generally infeasible for standard simulators. However, as shown by \citet{FlowDAS}, \citet{Filtering_GenCast} and \citet{DAISI}, this distribution can be accessed using diffusion emulators. 
    
    The key idea is to incorporate the observation $y^{k+1}$ into the score when solving the reverse diffusion Equation~\eqref{eq:reverse_diffusion_equation}, that is, using the score $\nabla_{\!x^{k+1}_{t}} \log p_{t} (x^{k+1}_{t} \mid x^{k}, y^{k+1})$ of the posterior. Thanks to Bayes' rule, this score can be decomposed as
    \begin{equation} \label{eq:score_bayes_diffusion}
        s^{x,y}_{t}(x_{t}^{k+1}, x^{k}, y^{k+1}) = s^{x}_{t}(x_{t}^{k+1}, x^{k}) + s^{y}_{t}(x_{t}^{k+1}, x^{k}, y^{k+1}),
    \end{equation}
    where $s^{x,y}_{t}$ denotes the score of the posterior, $s^{x}_{t}$ the score of the prior, and $s^{y}_{t}$ the score of the likelihood. Since the score of the prior is already available from the trained denoiser of the emulator through Equation~\eqref{eq:first_order_Tweedie}, the only unknown quantity that remains to be computed is the score $s^{y}_{t}(x_{t}^{k+1}, x^{k}, y^{k+1})$ of the likelihood.
    
    To do so, we use moment matching posterior sampling (MMPS, \citealp{MMPS}), a method that approximates the likelihood by
    \begin{align} \label{eq:likelihood_approximation_mmps}
        &p(y^{k+1} \mid x_{t}^{k+1}, x^{k}) \\
        &= \int p(y^{k+1} \mid x^{k+1})p(x^{k+1} \mid x_{t}^{k+1}, x^{k}) \mathrm{d}x^{k+1} \\        
        &\approx \int p(y^{k+1} \mid x^{k+1}) q(x^{k+1} \mid x_{t}^{k+1}, x^{k}) \mathrm{d}x^{k+1} \\
        & = \mathcal{N}(y^{k+1} \mid \mathcal{H}(\mathrm{m}), \Sigma_{y} + \mathrm{H}\mathrm{V}\mathrm{H}^{\top}),
    \end{align}
    where $ q(x^{k+1} \mid x_{t}^{k+1}, x^{k})$ is the density of Gaussian random variable with mean $\mathrm{m} = \mathbb{E}[x^{k+1} \mid x_{t}^{k+1}, x^{k}]$ and covariance $\mathrm{V} = \mathbb{V}[x^{k+1} \mid x_{t}^{k+1}, x^{k}]$, and $\mathrm{H}$ the Jacobian matrix of the observation operator. Assuming that the derivative of $\mathbb{V}[x^{k+1} \mid x^{k+1}_{t}, x^{k}]$ with respect to $x^{k+1}_{t}$ is negligible, we can then estimate the score of the likelihood as
    \begin{equation} \label{eq:score_likelihood}
        \nabla_{\!x_{t}^{k+1}}^{\top}(\mathrm{m})\mathrm{H}^{\top} (\Sigma_{y} + \mathrm{H}\mathrm{V}\mathrm{H}^{\top})^{-1}\left[y^{k+1} - \mathcal{H}(m) \right].
    \end{equation}
    The covariance matrix $\mathrm{V} = \mathbb{V}[x^{k+1} \mid x_{t}^{k+1}, x^{k}]$ is unknown a priori, but can be computed with the denoiser using the second-order Tweedie's formula (see Appendix \ref{appendix:Tweedie}) as
    \begin{equation} \label{eq:Tweedie_2nd_ordre}
        \mathbb{V}[x^{k+1} \mid x_{t}^{k+1}, x^{k}] = \alpha_{t}^{-1} \Sigma_{t} \nabla_{\!x_{t}^{k+1}}^{\top} (\mathrm{m}).
    \end{equation}
    Since the Jacobian $\nabla_{\!x_{t}^{k+1}}^{\top}d_{\theta}(t, x^{k}, x^{k+1}_{t}) \in \mathbb{R}^{n \times n}$ is intractable in high dimension and $(\Sigma_{y} + \mathrm{H}\mathrm{V}\mathrm{H}^{\top})$ is symmetric positive definite, we consider only implicit access to the covariance matrix via automatic differentiation and we solve the linear system in Equation~\eqref{eq:score_likelihood} using a linear solver \cite{GMRES, BiCG}. 

    Putting all these elements together, we can finally compute the score of the posterior and plug it into Equation~ \eqref{eq:reverse_diffusion_equation} to generate samples from the optimal proposal $p(x^{k+1} \mid x^{k}, y^{k+1})$. The resulting procedure is summarized in Algorithm~\ref{algo:sampling_optimal_proposal}.

    \begin{algorithm}[h!]
        \caption{Sampling from the optimal proposal}
        \begin{algorithmic}
            \STATE {\bfseries Inputs:} $x^{k}$, $y^{k+1}$, $d_{\theta}$, $\Delta_{t}$, Solver

            \vspace{0.2cm}

            \STATE $x^{k+1}_{t=1} \sim \mathcal{N}(0, \Sigma_{1})$

            \vspace{0.2cm}
                
            \FOR{$t$ in $[1, 1 - \Delta_{t}, \cdots, \Delta_{t}]$}

                \vspace{0.2cm}

                \STATE \quad $s_{x} \gets \Sigma_{t}^{-1} \left( \alpha_{t} d_{\theta}(t, x^{k}, x^{k+1}_{t}) - x^{k+1}_{t} \right)$~(Eq. \ref{eq:first_order_Tweedie})

                \vspace{0.2cm}
            
                \STATE \quad $s_{y} \gets \mathrm{MMPS}(t, x_{t}^{k+1}, x^{k}, y^{k+1})$~(Eq. \ref{eq:score_likelihood})

                \vspace{0.2cm}
                                
                \STATE \quad $s_{x,y} \gets s_{x} + s_{y}$~(Eq. \ref{eq:score_bayes_diffusion})

                \vspace{0.2cm}
                
                \STATE \quad $x^{k+1}_{t-\Delta_{t}} \gets \mathrm{Solver}(t, \Delta_{t}, x^{k+1}_{t}, s_{x,y}) $~(Eq. \ref{eq:reverse_diffusion_equation})

                \vspace{0.2cm}
            
            \ENDFOR

            \vspace{0.2cm}
                
            \STATE {\bfseries Return} $x^{k+1}_{t=0}$ 
        \end{algorithmic}
        \label{algo:sampling_optimal_proposal}
    \end{algorithm}
        
    \subsection{Computing weights} \label{subsection: weights}
        The final ingredient required to apply FA-APF is the computation of particle weights. When using the optimal proposal $q(x^{k+1} \mid x^{k}, y^{k+1}) = p(x^{k+1} \mid x^{k}, y^{k+1})$, the unnormalized weights in Equation~\eqref{eq:unnormalized_weights} simplify to
        \begin{equation} \label{eq:weight_optimal_proposal}
            \hat{w}^{k+1}_{i} = p(y^{k+1} \mid x^{k}_{i}) \times w^{k}_{i},
        \end{equation}
        where $w^{k}_{i}$ denotes the normalized weight of particle $i$ at time step $k$. 
        
        This quantity is not directly available, since the observation $y^{k+1}$ does not depend explicitly on $x^{k}$. We therefore approximate the transition distribution $p(x^{k+1} \mid x^{k})$ by a Dirac mass at its conditional mean $\mathbb{E}[x^{k+1} \mid x^{k}]$ \cite{Proba}. This yields
        \begin{align} \label{eq:weight_approximation}
             &p(y^{k+1} \mid x^{k}_{i}) \\
             &= \int p(y^{k+1} \mid x^{k+1}) p(x^{k+1} \mid x_{i}^{k}) \mathrm{d}x^{k+1} \\
             & \approx p(y^{k+1} \mid \mathbb{E}[x^{k+1} \mid x_{i}^{k}]) \\
             &= \mathcal{N}(\mathcal{H}(\mathbb{E}[x^{k+1} \mid x_{i}^{k}]), \Sigma_{y}).
        \end{align}

        For diffusion emulators, the conditional mean $\mathbb{E}[x^{k+1} \mid x^{k}_{i}]$ can be directly computed with one call to the trained denoiser as
        \begin{align} \label{eq:next_state_expectation}
        \mathbb{E}[x^{k+1} \mid x^{k}]  &\underset{\varepsilon \sim \mathcal{N}(0,I)}{=}  \mathbb{E}[x^{k+1} \mid x^{k}, \sigma_{1}\varepsilon] \\
        &\approx d_{\theta}\left(x^{k+1}_{t=1} = \sigma_{1}\varepsilon, x^{k}, t = 1 \right).
        \end{align}
        This approximation provides an efficient estimate of the particle weights and corresponds to lines 5–10 of Algorithm~\ref{algo:FA-APF}.

\section{Experiments}
        
    \subsection{Lorenz’63 System} \label{subsection: Lorenz63}
    We first evaluate our method on Lorenz'63 (L63), a simple yet widely used system to assess filtering performance. In one of its stochastic formulation \cite{L63_Stochastic}, this system evolves according to
    \begin{equation}
    \begin{cases}
            \dot{x} =  s(y - x) + x\sigma dw_{t}, \\
            \dot{y} = x(r - z) - y + y\sigma dw_{t}, \\
            \dot{z} = xy - bz + z\sigma dw_{t},
        \end{cases}
    \end{equation}
    where $(s, r, b)$ are the parameters of the dynamic, and $\sigma$ a parameter controlling the level of stochasticity. We set $(s=10, r=28, b=8/3)$ to ensure a chaotic regime, typical of the systems encountered in practice, and $\sigma = 0.25$.

    Numerical integration is performed using a Milstein scheme \cite{Milstein} with a time step of $10^{-3}$, and we collect snapshots of the system states every $0.5$ time unit to create a dataset of successive states $\{(x^{k}_{i}, x^{k+1}_{i})\}_{i \in \mathcal{I}}$. The latter contains $10 000$ trajectories of $100$ snapshots each, and is used to train a denoiser with the loss of Equation~\eqref{eq:training_conditional_diffusion}.

    The denoiser is a residual network \cite{resnet} with six hidden layers and follows the formalism of \citet{EDM}. In particular, we adopt the same noise scheduling, with $\alpha_{t} = 1$ and $\sigma_{t}$ given by
    \begin{equation} \label{eq:noise_scheduling_EDM}
        \sigma_{t} = \left( \sigma_{\mathrm{max}}^{1/\rho} + (1-t)  \times \left( \sigma_{\mathrm{min}}^{1/\rho} - \sigma_{\mathrm{max}}^{1/\rho}\right)\right)^{\rho},
    \end{equation}
    where $\rho = 7$, $\sigma_{\mathrm{min}} = 10^{-3}$ and $\sigma_{\mathrm{max}} = 10^{3}$. The denoiser is trained for 10 epochs and, at inference time, we use a DDIM sampler \cite{DDIM} with 64 steps to solve the reverse diffusion Equation~\eqref{eq:reverse_diffusion_equation} and generate samples from the optimal proposal (see Algorithm \ref{algo:sampling_optimal_proposal}).

    We study the ability of our method to approximate the filtering distribution as a function of the number of particles ($N$ in Algorithm \ref{algo:FA-APF}), also called number of ensemble members in the classical data assimilation literature \cite{data_assimilation}. This analysis is essential for practical applications, as high-dimensional systems typically prevent the use of large ensembles because of computational cost \cite{num_ens_members}. We consider observations of the first and last component of the system, corrupted by Gaussian noise with standard deviation $0.25$. This formally corresponds to define the observation $y^{k}$ in Equation~\eqref{eq:observation_model} as
    \begin{equation}
        y^{k} = \underbrace{\begin{bmatrix} 1 & 0 & 0 \\ 0 & 0 & 1  \end{bmatrix} \begin{bmatrix} x \\ y \\ z \end{bmatrix}}_{\mathcal{H}(x^{k})} + \underbrace{\begin{bmatrix} (0.25)^{2} & 0\\ 0 & (0.25)^{2}  \end{bmatrix}  \varepsilon}_{\Sigma_{y}},
    \end{equation}
    with $\varepsilon \sim \mathcal{N}(0,I_{2})$. We compare our approach with other ensemble methods that operate directly using the transition model and the observation operator to approximate the filtering distribution at each step $k$ by an ensemble of particles. These methods are described in Appendix \ref{appendix:DA_algo}. For each method and ensemble size, we perform $32$ filtering runs and report the average skill, defined as the root mean square error (RMSE) between the ensemble mean and the true state $x^{k}$, on Figure \ref{fig:skill_particles}.
    
    \vspace{-0.25in} 
    \begin{figure}[h]
      \vskip 0.2in
      \begin{center}
        \centerline{\includegraphics[width=\columnwidth]{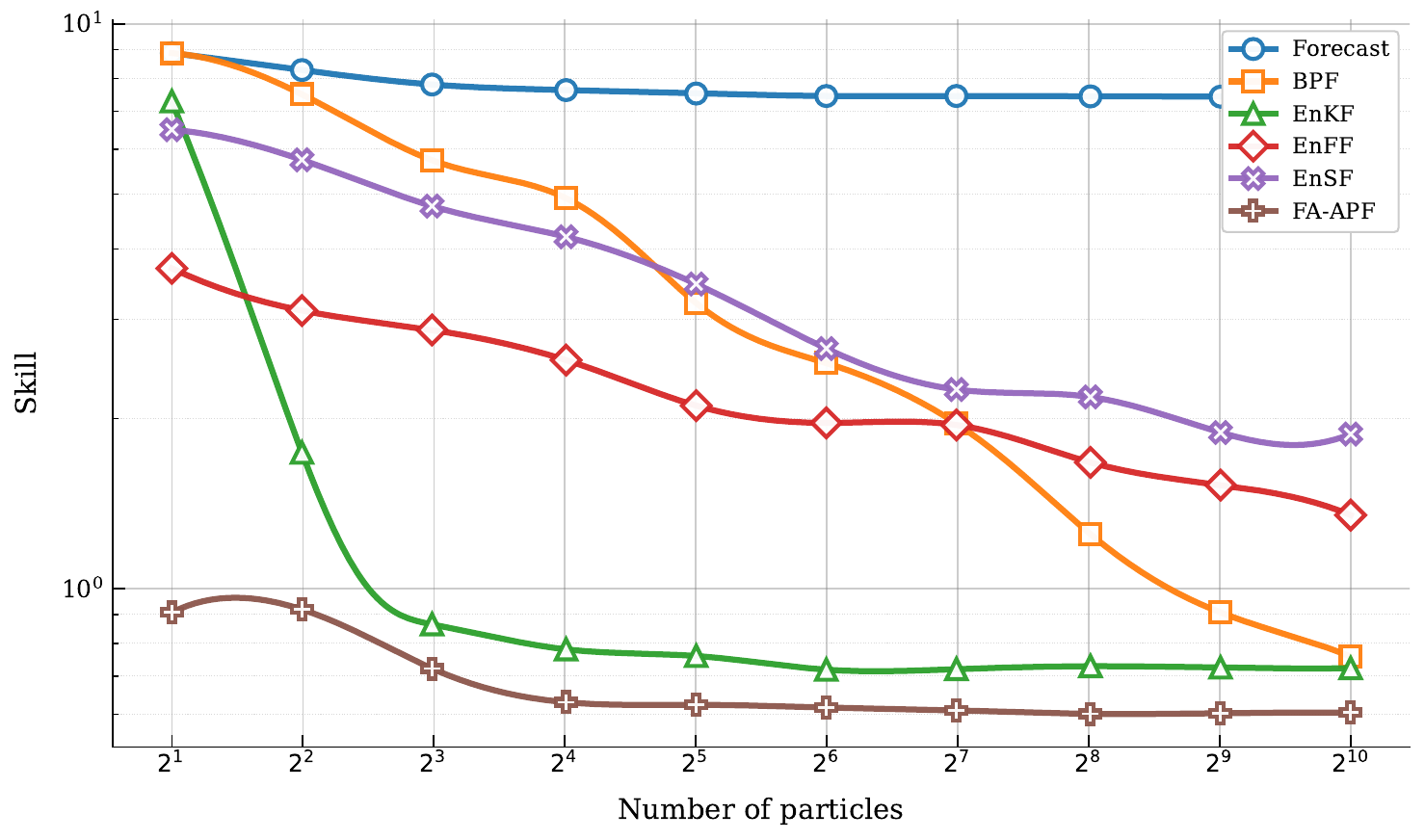}}
        \caption{
          Evolution of the average skill as a function of the number of ensemble members.
        }
        \label{fig:skill_particles}
      \end{center}
    \end{figure}
    \vspace{-0.25in}  
    
    Our method (brown curve on Figure \ref{fig:skill_particles}) consistently yields lower errors than competing algorithms for a given number of members, and outperforms them even when using fewer members. In particular, FA-APF substantially outperforms the classical particle filter (BPF, orange curve on Figure \ref{fig:skill_particles}), which is known to require many particles for reliable performance. This improvement is due to the use of the optimal proposal in FA-APF (whereas the classical BPF relies on the standard proposal $q(x^{k+1} \mid x^{k}, y^{k+1}) = p(x^{k+1} \mid x^{k})$), and highlights the potential of particle filters for high-dimensional applications. We also note the strong performance of EnKF \cite{EnKF_Evensen} with a limited number of particles, explaining its widespread use in operational meteorology centers. Table \ref{table:experiments_lorenz63} reports the results for all algorithms with a fixed ensemble size of $N=256$. The first five rows refer to the previous experiment for which we observe the first and last components, while the last five rows correspond to a second experiment where we only observe the first component with a standard deviation of 0.25.

    \begin{table}[h]
      \caption{Average skill, spread-to-skill ratio (SSR) and CRPS (see Appendix \ref{appendix:metrics}) on 32 filtering runs for each method.}
      \begin{center}
        \begin{small}
          \begin{sc}
            \begin{tabular}{lcccr}
              \toprule
                    & Skill $(\downarrow)$ & SSR $(\approx 1)$ & CRPS $(\downarrow)$  \\
              \midrule
              BPF  & $1.25$ & $0.32$ & $2.88$  \\
              EnKF & $\underline{0.72} $ & $\underline{1.25}$ & $\underline{1.10}$ \\
              EnSF & $2.18$ & $\mathbf{0.78}$ & $4.23$ \\
              EnFF & $1.67$ & $\underline{1.25}$ & $2.51$ \\
              FA-APF & $\mathbf{0.6}$ & 1.49 & $\mathbf{0.98}$ \\
              \midrule
              BPF  & $\underline{2.38}$ & $\mathbf{0.85}$ & $\underline{4.35} $  \\
              EnKF & $2.79$ & $1.21$ & $\underline{4.35}$ \\
              EnSF & $3.76$ & $\underline{1.19}$ & $6.32$ \\
              EnFF & $3.67$ & $1.39$ & $5.96$ \\
              FA-APF & $\mathbf{1.87}$ & 1.25 & $\mathbf{3.05}$ \\
              \bottomrule
            \end{tabular}
          \end{sc}
        \end{small}
      \end{center}
      \vskip -0.1in
      \label{table:experiments_lorenz63}
    \end{table}

    \newpage
    \subsection{Incompressible Navier-Stokes Flow} \label{subsection: NS}
    Another important aspect in practice is the robustness of the algorithm to sparse observations, that is, when the observation space has a much lower dimension than the state space. For example, methods such as EnSF \cite{EnSF} and EnFF \cite{EnFF}, although theoretically elegant, perform significantly worse in such regimes \cite{Latent_EnSF}.

    To evaluate our method under these conditions, we consider a high-dimensional system describing an incompressible fluid governed by the 2D Navier–Stokes equations with random forcing on the torus $\mathbb{T}^{2} = [0, 2\pi]^{2}$
    \begin{equation}
        \mathrm{d}\mathbf{\omega} + v \cdot \nabla \omega ~\mathrm{d}t = \nu \Delta \omega ~ \mathrm{d}t - \alpha \omega~ \mathrm{d}t + \varepsilon \mathrm{d}\xi,
    \end{equation}
    where $\omega$ represents the vorticity, $v$ the velocity, and $\xi$ a white-in-time random forcing acting on Fourier modes. We directly rely on the dataset of \citet{probabilistic_forecasting_interpolants}, which consists of $\mathcal{O}(10^{5})$ consecutive vorticity snapshots sampled every 0.5 time units on a $256 \times 256$ grid, that we downsample to $128\times128$ for computational efficiency.

    We use this dataset to train a denoiser by minimizing the loss function defined in Equation~\eqref{eq:training_conditional_diffusion}, using the same formalism and noise scheduler as the previous experiment on Lorenz'63. The backbone of the denoiser is a U-Net \cite{U-Net} with $\mathcal{O}(10^{7})$ parameters, and is trained for 20 epochs. At inference time, we use the same DDIM sampler with only 32 steps to accelerate generation when solving Equation~\eqref{eq:reverse_diffusion_equation}.

    For the experiments, we consider a subsampling operator that selects specific points of the grid, and a coarsening operator that pixelates system states (see Figure \ref{fig:example_filtering_ns}). Observations are corrupted with Gaussian noise of standard deviation $0.1$. We consider observations at three dimensional levels, corresponding to the rows of Table \ref{table:experiments_ns}. For instance, $(32,32)$ indicates an observation with dimension $32 \times 32$, corresponding to $6.25\%$ of the full system state. Table \ref{table:experiments_ns} reports the average skill over ten filtering experiments for the coarsening operator (first three rows) and the subsampling operator (last three rows) for BPF, FA-APF, and FlowDAS \cite{FlowDAS} with $128$ particles.

    \begin{table}[h]
      \caption{Average skill for BPF, FlowDAS and FA-APF.}
      \begin{center}
        \begin{small}
          \begin{sc}
            \begin{tabular}{c|c|ccc}
              \cmidrule[\heavyrulewidth]{3-5} 
              \multicolumn{1}{c}{} & & \multicolumn{3}{c}{Algorithms} \\
                        \cmidrule{1-2} \cmidrule(lr){3-5}
              $\mathcal{H}$ & $d$ & BPF & FlowDAS & FA-APF \\
              \midrule
              \multirow{3}{*}{coarse} 
              &$(8,8)$ & $3.02$ & $\underline{2.78}$ & $\mathbf{2.39}$\\
              &$(16,16)$ & $2.88$ & $\underline{0.74}$ & $\mathbf{0.63}$\\
              &$(32,32)$ & $2.80$ & $\underline{0.19}$ & $\mathbf{0.14}$ \\
              \midrule
              \multirow{3}{*}{sparse} 
              &$(8,8)$  & $\underline{2.96}$ & $3.08$ & $\mathbf{2.30}$ \\
              &$(16,16)$ & $2.91$ & $\underline{1.9}$ & $\mathbf{1.12}$ \\
              &$(32,32)$  &$2.88$ & $\underline{0.20}$ &  $\mathbf{0.13} $\\
              \bottomrule
            \end{tabular}
          \end{sc}
        \end{small}
      \end{center}
      \vskip -0.1in
      \label{table:experiments_ns}
    \end{table}
    
    \newpage
    We consistently observe a lower skill with FA-APF and strong qualitative results as in Figure \ref{fig:example_filtering_ns}, especially compared to the classical particle filter (BPF), which fails entirely under such high-dimensional settings, confirming that it is unusable in realistic geophysical systems \cite{pf_degeneracy}.

    Our method is also more robust than FlowDAS, which can be considered as training-free if an interpolant of the dynamics is already available \cite{probabilistic_forecasting_interpolants}. Actually, as presented by \citet{FlowDAS}, FlowDAS corresponds to a degenerate fully adapted filter that propagates its single particle sequentially using the optimal proposal. Extending this approach to an ensemble of particles yields a pseudo FA-APF that ignores particle weights and does not select which particles to propagate at each step, resulting in a distribution that deviates from the true filtering distribution. Our algorithm can therefore be seen as a natural ensemble-based extension of FlowDAS, while remaining rigorously grounded in the particle filtering formalism.
    
    \begin{figure}[h]
      \begin{center}
\centerline{\includegraphics[width=0.9\columnwidth]{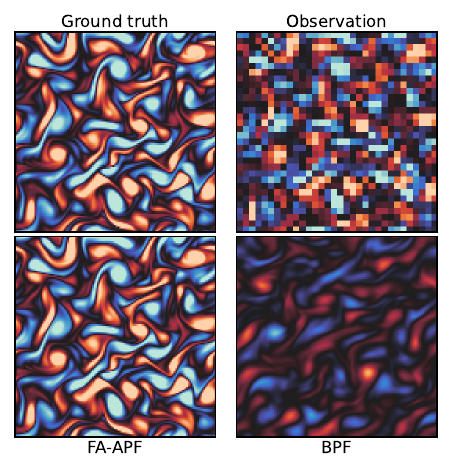}}
        \caption{
          Ground truth, coarse $32 \times 32$ observation, FA-APF ensemble mean and BPF ensemble mean at the last step of a filtering experiment. Examples of trajectories are given in Appendix \ref{appendix:additional_results}.  
        }
        \vskip -0.1in
        \label{fig:example_filtering_ns}
      \end{center}
    \end{figure}

    Like FlowDAS, our method can be adapted to the stochastic interpolant framework \cite{Interpolants}, as detailed in Appendix \ref{appendix:stochastic interpolants}. This framework corresponds to a family of generative models that generalize diffusion and flow matching \cite{FlowMatching}, and can thus be applied to a wide range of generative emulators. For clarity, we adopted the diffusion framework in this study, as it is the most established and widely used framework, and it was also the framework behind GenCast \cite{GenCast}, which we examined in the next experiment.

    \begin{figure*}[ht]
        \begin{center}
            \centerline{\includegraphics[width=0.95\textwidth]{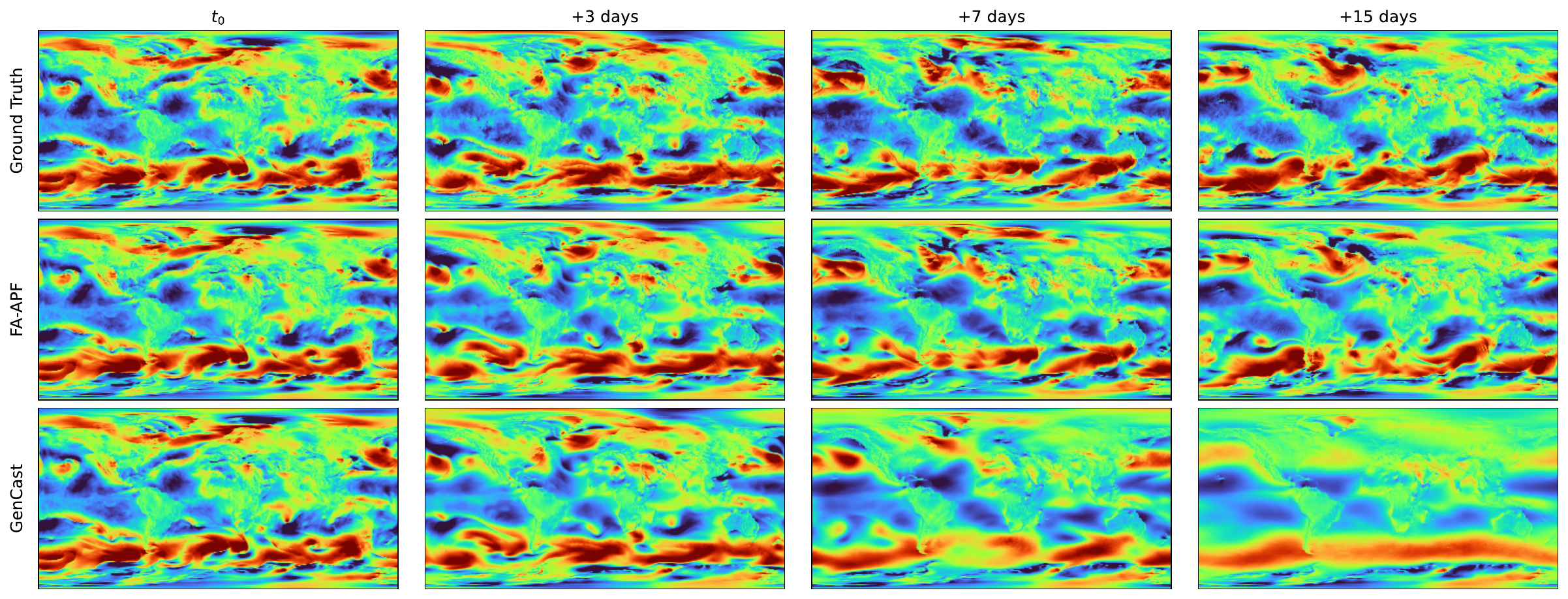}}
            \caption{
              Comparison of the 10m U component of wind between the reference ERA5 trajectory (first row), the FA-APF ensemble mean obtained with realistic observations (second row), and the GenCast ensemble mean (third row) after 3, 7, and 15 days.
            }
            \vskip -0.2in
            \label{fig:10m_U_component_of_wind}
        \end{center}
    \end{figure*}

    \newpage
    \vspace{-0.25in} 
    \subsection{Medium-range weather forecasts (GenCast)} \label{subsection: GenCast}
    In the final experiment, we apply our method in a real-world scenario by leveraging the denoiser from GenCast \cite{GenCast}, a diffusion emulator of the atmosphere trained on the ERA5 reanalysis dataset \cite{ERA5}. In this setting, the system state $x^{k}$ is high-dimensional, as it consists of $83$ surface and atmospheric variables defined on a 1° latitude-longitude grid, resulting in $\mathcal{O}(10^{6})$ variables.

    An ensemble of $N=256$ particles is initialized using the first state of a reference ERA5 trajectory that was not included in the training dataset. This trajectory is then used as ground truth to evaluate the metrics and to generate observations over a 15-day period, corresponding to 30 time steps with the 12-hour resolution of GenCast. We study two observation scenarios:
    \begin{itemize}
        \item \textbf{Sparse temperature observations}. 
        We subsample and coarsen the latitude–longitude grid by retaining one point out of 4 in each direction and averaging $4 \times 4$ non-overlapping patches. We observe only temperature, the most readily available variable in practice from weather stations and satellites. This corresponds to observing approximately 1\% of the full system state. Gaussian observation noise with standard deviation 0.2 is added to reflect modern temperature measurement accuracy.
        \item \textbf{Realistic observations}.
        We consider ground-based stations inspired by real-world station locations. These stations measure 4 of the 6 surface variables with Gaussian noise based on the performance of current measuring instruments. In addition, drawing inspiration from the setting of \citet{Appa}, we include satellite observations of atmospheric temperature bands with Gaussian noise of 0.5K, reflecting the indirect nature of temperature retrieval from radiance measurements.
    \end{itemize}

    Figure \ref{fig:skill_gencast} shows the evolution of the skill for two surface variables (U component of wind and temperature) across successive assimilation steps. Results are shown for our filtering method under both observation scenarios and for an ensemble of forecasts with the same ensemble size.

    \vspace{-0.2in}
    \begin{figure}[h]
      \vskip 0.2in
      \begin{center}
        \centerline{\includegraphics[width=\columnwidth]{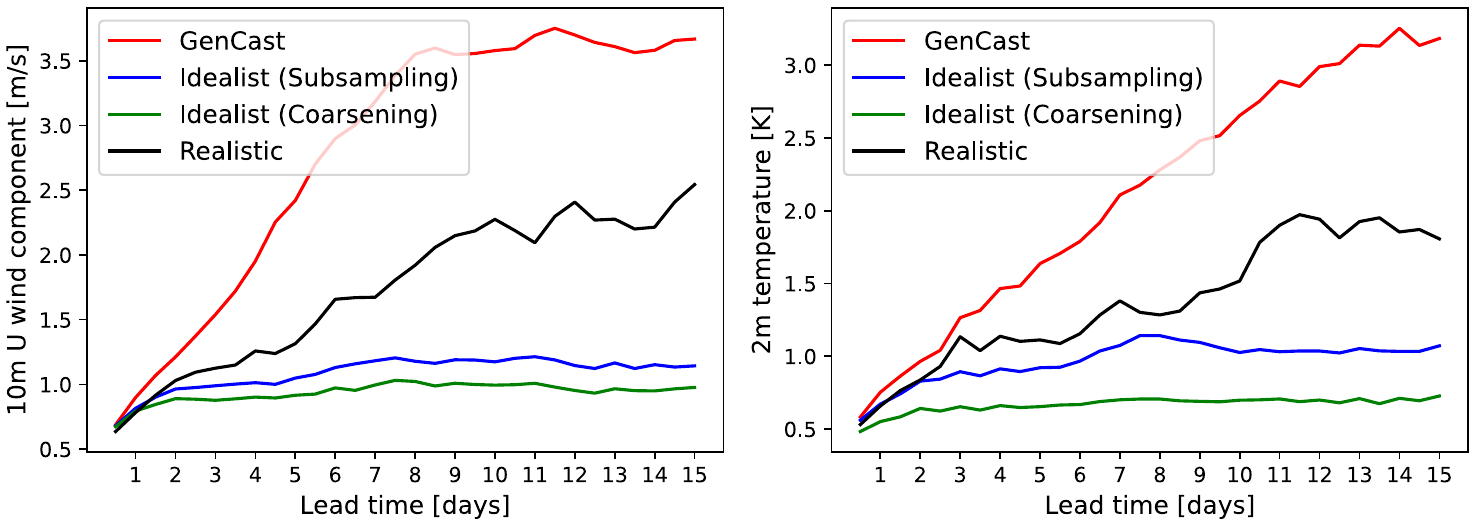}}
        \caption{
          Skill comparison between FA-APF with sparse temperature observations (blue and green curves), FA-APF with realistic observations (black curve), and the ensemble of unconditional GenCast trajectories (red curve) for the surface variables.
        }
        \label{fig:skill_gencast}
      \end{center}
    \end{figure}
    \vspace{-0.2in}

    For the first observation setting, based on subsampled and coarsened temperature observations, we obtain a stable approximation of the system state after 7 days of observations. This holds even for unobserved variables, such as wind in Figure \ref{fig:skill_gencast}, while maintaining a non-zero ensemble spread. Spread and additional skill curves are given in Appendix \ref{appendix:additional_results}.

    Under the more realistic observation setting, convergence toward a stable skill across all variables is not achieved, even after 15 days of observations. This is mainly due to the strong spatial inhomogeneity of the observations, both at the surface and in the atmosphere, particularly near the poles, where errors grow rapidly in the absence of measurements. Nevertheless, this setting has intentionally been made challenging, and operational systems would typically have access to denser observations. Figure \ref{fig:10m_U_component_of_wind} illustrates the ensemble mean obtained by filtering in this setting, which remains qualitatively close to the true system state.

\newpage
\section{Related work} \label{section:Related_Work}
    Variational methods such as 4D-Var \cite{4D_VAR} have been widely used in operational centers \cite{Arome3DEnVAR} and are effective in practice. However, they rely on tangent linear and adjoint models, which are computationally expensive and may fail to capture strongly nonlinear dynamics, such as extreme weather events in the case of the atmosphere. Moreover, they provide only point estimates rather than full probabilistic predictions, preventing uncertainty quantification.

    Ensemble Kalman filters (EnKF) and their variants, such as LETKF \cite{LETKF}, are another widely used class of methods in operational weather centers \cite{Arome3DEnVAR}. They provide probabilistic state estimates at each assimilation step but are theoretically valid only under linear dynamics, linear observation operator, and Gaussian noise assumptions. In practice, heuristics such as covariance inflation and localization are needed to handle nonlinear, high-dimensional systems. Despite these limitations, our experiments show that EnKF variants remain competitive.

    More recently, training-free methods based on generative models have been introduced \cite{EnSF, EnFF}. These methods rely only on the transition model $p(x^{k+1} \mid x^{k})$ and the observation operator $\mathcal{H}$. Starting from the filtering distribution at time $k$, they forecast particles to the next time step and estimate a vector field that maps the current filtering distribution to the next predictive distribution. This vector field is then adjusted using gradients of the observation operator to converge toward the next filtering distribution. While elegant in theory, these approaches are effective only under dense observation regimes \cite{Latent_EnSF}, which are rare in practice.

    Other generative approaches require training of specific models. The Score-Based Data Assimilation framework \cite{SDA} focuses on Bayesian smoothing by training a local score network and combining these local scores to generate a full trajectory consistent with the observations. The DAISI framework \cite{DAISI} proposes an iterative filtering algorithm based on a stochastic interpolant learned directly from the data. However, this method is not guaranteed to correspond to proper Bayesian filtering and depends on the efficiency of the forecast model.

    Finally, our work is closely related to FlowDAS \cite{FlowDAS}, which can be viewed as a training-free method adaptable to generative emulators of dynamical systems. However, FlowDAS is also not mathematically grounded, as its algorithm does not provide an exact solution to the Bayesian filtering problem, except in the trivial case of a single particle, thereby losing the advantage of ensemble methods for uncertainty quantification. Our method can thus be seen as a generalization of FlowDAS to proper Bayesian filtering.

\section{Conclusion}
In this work, we showed that generative emulators of dynamical systems can be adapted, without additional training, to address Bayesian filtering through an optimized version of particle filters \cite{FA_APF}. Although this variant is known to be more efficient than the classical particle filter \cite{Optimal_Proposal_Efficiency}, it has remained largely impractical due to the difficulty of sampling from the optimal proposal, a limitation we address using training-free posterior sampling such as MMPS \cite{MMPS}.

Our results demonstrate that this approach consistently outperforms the classical particle filter and competing methods for a fixed ensemble size (Section \ref{subsection: Lorenz63}), and remains effective in high-dimensional settings with sparse observations (Sections \ref{subsection: NS} and \ref{subsection: GenCast}), successfully scaling to problems with up to $\mathcal{O}(10^6)$ variables. These findings show that our method can be successfully applied to realistic large-scale problems without relying on linearization or restrictive assumptions on the system dynamics.

\section{Limitations \& Future work}
A fundamental limitation of our approach is its reliance on generative emulators of the system dynamics. While such models are becoming increasingly popular \cite{lola, DiffusionLAM, Diffusion_Sea_Ice, GenCast}, in particular because they preserve uncertainty and avoid long-term over-smoothing, they are still less used than classical numerical solvers and deterministic neural emulators \cite{IFS, GraphCast}.

Sampling from the optimal proposal is also computationally expensive, as it requires differentiating the denoiser at each step of the reverse diffusion process to estimate the score of the posterior. For very large-scale systems such as GenCast (Section \ref{subsection: GenCast}), this results in significant memory overhead (see Appendix~\ref{appendix:details_algo} for further details on the complexity of Algorithm \ref{algo:FA-APF}). An important direction for future work is to develop training-free posterior sampling methods that remain accurate for highly nonlinear observation operators while being computationally efficient.

The performance of our method is further constrained by the quality of the transition model, the approximations used to sample from the optimal proposal and compute particle weights, and the use of inflation in Algorithm \ref{algo:FA-APF} to mitigate weight degeneracy. Additional details on these approximations are given in Appendix \ref{appendix:details_algo}.

Finally, assuming access to generative models capable of directly sampling joint distributions $p(x^{k}, x^{k+1})$, a natural extension of this work would be to adapt the proposed framework for training-free Bayesian smoothing \cite{Smoothing_Doucet, Particle_Smoothing}.

\newpage

\section*{Acknowledgments and Disclosure of Funding}
We acknowledge the support of the F.R.S.-FNRS (Belgium) and its funding of the Mosaic project (MIS F.4536.25). François Rozet is a research fellow of the F.R.S.-FNRS and acknowledges its financial support.

The present research benefited from computational resources made available on Lucia, the Tier-1 supercomputer of the Walloon Region, infrastructure funded by the Walloon Region under the grant agreement n°1910247. We also acknowledge the support of NVIDIA Corporation for computing resources offered through the Academic Grant Program.

\section*{Impact Statement}
This paper presents work whose goal is to advance the fields of Machine Learning and Data Assimilation. There are many potential societal consequences of our work, none of which we feel must be specifically highlighted here.

\bibliography{references}

@article{Lorenz63,
    author = "Edward N.  Lorenz",
    title = "{D}eterministic {N}onperiodic {F}low",
    journal = "Journal of Atmospheric Sciences",
    year = "1963",
    publisher = "American Meteorological Society",
    address = "Boston MA, USA",
    volume = "20",
    number = "2",
    doi = "10.1175/1520-0469(1963)020<0130:DNF>2.0.CO;2",
    pages="130 - 141",
    url = "https://journals.ametsoc.org/view/journals/atsc/20/2/1520-0469_1963_020_0130_dnf_2_0_co_2.xml"
}

@article{NavierStokesSimulation,
    ISSN = {00255718, 10886842},
    URL = {http://www.jstor.org/stable/2004575},
    author = {Alexandre Joel Chorin},
    journal = {Mathematics of Computation},
    number = {104},
    pages = {745--762},
    publisher = {American Mathematical Society},
    title = {{N}umerical {S}olution of the {N}avier-{S}tokes {E}quations},
    urldate = {2025-08-14},
    volume = {22},
    year = {1968}
}

@book{Hairer,
  address = {Berlin},
  author = {Hairer, E. and N{\o}rsett, S.P. and Wanner, G.},
  edition = {Second},
  publisher = {Springer},
  title = {Solving Ordinary Differential Equations {I} Nonstiff problems},
  year = 2000
}

@InProceedings{cnn_simulation,
    title = {{A}ccelerating {E}ulerian {F}luid {S}imulation {W}ith {C}onvolutional {N}etworks},
    author = {Jonathan Tompson and Kristofer Schlachter and Pablo Sprechmann and Ken Perlin},
    booktitle = {Proceedings of the 34th International Conference on Machine Learning},
    pages = {3424--3433},
    year = {2017},
    editor = {Precup, Doina and Teh, Yee Whye},  volume = {70}, 
    series = {Proceedings of Machine Learning Research},  month =  {06--11 Aug},  publisher =    {PMLR},
    url = {https://proceedings.mlr.press/v70/tompson17a.html}
}

@article{GraphCast,
    author = {Remi Lam  and Alvaro Sanchez-Gonzalez  and Matthew Willson  and Peter Wirnsberger  and Meire Fortunato  and Ferran Alet  and Suman Ravuri  and Timo Ewalds  and Zach Eaton-Rosen  and Weihua Hu  and Alexander Merose  and Stephan Hoyer  and George Holland  and Oriol Vinyals  and Jacklynn Stott  and Alexander Pritzel  and Shakir Mohamed  and Peter Battaglia },
    title = {Learning skillful medium-range global weather forecasting},
    journal = {Science},
    volume = {382},
    number = {6677},
    pages = {1416-1421},
    year = {2023},
    doi = {10.1126/science.adi2336},
    URL = {https://www.science.org/doi/abs/10.1126/science.adi2336},
    eprint = {https://www.science.org/doi/pdf/10.1126/science.adi2336}
}

@inproceedings{diffusion,
    author = {Song, Yang and Sohl-Dickstein, Jascha and Kingma, Diederik P and Kumar, Abhishek and Ermon, Stefano and Poole, Ben},
    booktitle = {International Conference on Learning Representations},
    title = {{S}core-{B}ased {G}enerative {M}odeling through {S}tochastic {D}ifferential {E}quations},
    url = {https://openreview.net/forum?id=PxTIG12RRHS},
    year = 2021
}

@inproceedings{FlowMatching,
    title={{F}low {M}atching for {G}enerative {M}odeling},
    author={Yaron Lipman and Ricky T. Q. Chen and Heli Ben-Hamu and Maximilian Nickel and Matthew Le},
    booktitle={The Eleventh International Conference on Learning Representations },
    year={2023},
    url={https://openreview.net/forum?id=PqvMRDCJT9t}
}

@article{Interpolants,
    author  = {Michael Albergo and Nicholas M. Boffi and Eric Vanden-Eijnden},
    title   = {{S}tochastic {I}nterpolants: {A} {U}nifying {F}ramework for {F}lows and {D}iffusions},
    journal = {Journal of Machine Learning Research},
    year    = {2025},
    volume  = {26},
    number  = {209},
    pages   = {1--80},
    url     = {http://jmlr.org/papers/v26/23-1605.html}
}

@InProceedings{probabilistic_forecasting_interpolants,
    title = {{P}robabilistic {F}orecasting with {S}tochastic {I}nterpolants and {F}öllmer {P}rocesses},
    author = {Chen, Yifan and Goldstein, Mark and Hua, Mengjian and Albergo, Michael Samuel and Boffi, Nicholas Matthew and Vanden-Eijnden, Eric},
    booktitle = {Proceedings of the 41st International Conference on Machine Learning},
    pages = {6728--6756},
    year = {2024},
    editor = {Salakhutdinov, Ruslan and Kolter, Zico and Heller, Katherine and Weller, Adrian and Oliver, Nuria and Scarlett, Jonathan and Berkenkamp, Felix},
    volume = {235},
    series = {Proceedings of Machine Learning Research},
    month = {21--27 Jul},
    publisher = {PMLR},
    url = {https://proceedings.mlr.press/v235/chen24n.html}
}

@inproceedings{lola,
    title={{L}ost in {L}atent {S}pace: {A}n {E}mpirical {S}tudy of {L}atent {D}iffusion {M}odels for {P}hysics {E}mulation},
    author = {Fran{\c{c}}ois Rozet and Ruben Ohana and Michael McCabe and Gilles Louppe and Fran{\c{c}}ois Lanusse and Shirley Ho},
    booktitle={The Thirty-ninth Annual Conference on Neural Information Processing Systems},
    year={2025},
    url={https://openreview.net/forum?id=xoNrbfbekM}
}

@book{intro_to_chaos,
    author = {Devaney, Robert L.},
    title = {An Introduction to Chaotic Dynamical Systems},
    edition = {2},
    year = {2003},
    publisher = {CRC Press},
    doi = {10.4324/9780429502309},
    url = {https://doi.org/10.4324/9780429502309}
}

@article{data_assimilation,
    author = {Carrassi, Alberto and Bocquet, Marc and Bertino, Laurent and Evensen, Geir},
    title = {Data assimilation in the geosciences: An overview of methods, issues, and perspectives},
    journal = {WIREs Climate Change},
    volume = {9},
    number = {5},
    pages = {e535},
    doi = {https://doi.org/10.1002/wcc.535},
    url = {https://wires.onlinelibrary.wiley.com/doi/abs/10.1002/wcc.535},
    eprint = {https://wires.onlinelibrary.wiley.com/doi/pdf/10.1002/wcc.535},
    year = {2018}
}

@book{EnKF_Evensen,
    address = {Berlin},
    author = {Evensen, Geir},
    doi = {10.1007/978-3-642-03711-5},
    edition = {Second},
    isbn = {978-3-642-03710-8},
    publisher = {Springer},
    title = {Data Assimilation : The Ensemble Kalman Filter},
    url = {https://link.springer.com/book/10.1007/978-3-642-03711-5},
    year = 2009
}

@article{lorenc_data_assimilation,
    author = {Lorenc, A. C.},
    title = {Analysis methods for numerical weather prediction},
    journal = {Quarterly Journal of the Royal Meteorological Society},
    volume = {112},
    number = {474},
    pages = {1177-1194},
    doi = {https://doi.org/10.1002/qj.49711247414},
    url = {https://rmets.onlinelibrary.wiley.com/doi/abs/10.1002/qj.49711247414},
    eprint = {https://rmets.onlinelibrary.wiley.com/doi/pdf/10.1002/qj.49711247414},
    year = {1986},
}

@inproceedings{SDA,
    title={{S}core-based {D}ata {A}ssimilation},
    author={Fran{\c{c}}ois Rozet and Gilles Louppe},
    booktitle={Thirty-seventh Conference on Neural Information Processing Systems},
    year={2023},
    url={https://openreview.net/forum?id=VUvLSnMZdX},
}

@inproceedings{DiffDA,
    author = {Huang, Langwen and Gianinazzi, Lukas and Yu, Yuejiang and Dueben, Peter D. and Hoefler, Torsten},
    title = {{D}iff{DA}: a {D}iffusion model for weather-scale {D}ata {A}ssimilation},
    year = {2024},
    publisher = {JMLR.org},
    booktitle = {Proceedings of the 41st International Conference on Machine Learning},
    articleno = {797},
    numpages = {18},
    location = {Vienna, Austria},
    series = {ICML'24}
}

@article{EnSF,
    title = {An ensemble score filter for tracking high-dimensional nonlinear dynamical systems},
    journal = {Computer Methods in Applied Mechanics and Engineering},
    volume = {432},
    pages = {117447},
    year = {2024},
    issn = {0045-7825},
    doi = {https://doi.org/10.1016/j.cma.2024.117447},
    url = {https://www.sciencedirect.com/science/article/pii/S0045782524007023},
    author = {Feng Bao and Zezhong Zhang and Guannan Zhang}
}

@book{BayesianFiltering, 
    place={Cambridge},
    series={Institute of Mathematical Statistics Textbooks},
    title={Bayesian Filtering and Smoothing},
    publisher={Cambridge University Press},
    author={Särkkä, Simo},
    year={2013},
    collection={Institute of Mathematical Statistics Textbooks}
}

@article{Arome3DEnVAR,
    author={Brousseau, P. and Vogt, V. and Arbogast, E. and Martet, M. and Thomas, G. and Berre, L.},
    title={The operational {3DE}n{V}ar data assimilation scheme for the {M}étéo-{F}rance convective scale model {AROME}-{F}rance},
    journal={EGUsphere [preprint]},
    year={2025},
    doi={10.5194/egusphere-2025-2642},
    url={https://doi.org/10.5194/egusphere-2025-2642}
}

@article{PF,
    author = {van Leeuwen, Peter Jan and Künsch, Hans R. and Nerger, Lars and Potthast, Roland and Reich, Sebastian},
    title = {Particle filters for high-dimensional geoscience applications: A review},
    journal = {Quarterly Journal of the Royal Meteorological Society},
    volume = {145},
    number = {723},
    pages = {2335-2365},
    doi = {https://doi.org/10.1002/qj.3551},
    url = {https://rmets.onlinelibrary.wiley.com/doi/abs/10.1002/qj.3551},
    eprint = {https://rmets.onlinelibrary.wiley.com/doi/pdf/10.1002/qj.3551},
    year = {2019}
}

@article{FA_APF,
    author  = {Petetin, Yohan and Desbouvries, François},
    title   = {{O}ptimal {SIR} algorithm vs. fully adapted auxiliary particle filter: a non asymptotic analysis},
    journal = {Statistics and Computing},
    year    = {2013},
    volume  = {23},
    number  = {6},
    pages   = {759--775},
    doi     = {10.1007/s11222-012-9345-5},
    url     = {https://doi.org/10.1007/s11222-012-9345-5},
    issn    = {1573-1375}
}

@article {HighDimensional_PF,
    author = "Chris Snyder and Thomas Bengtsson and Peter Bickel and Jeff Anderson",
    title = "{O}bstacles to {H}igh-{D}imensional {P}article {F}iltering",
    journal = "Monthly Weather Review",
    year = "2008",
    publisher = "American Meteorological Society",
    address = "Boston MA, USA",
    volume = "136",
    number = "12",
    doi = "10.1175/2008MWR2529.1",
    pages=      "4629 - 4640",
    url = "https://journals.ametsoc.org/view/journals/mwre/136/12/2008mwr2529.1.xml"
}

@article{GenCast,
    author  = {Price, I. and Sanchez-Gonzalez, A. and Alet, F. and Andersson, T. R. and El-Kadi, A. and Masters, D. and Ewalds, T. and Stott, J. and Mohamed, S. and Battaglia, P. and Lam, R. and Willson, M.},
    title   = {Probabilistic weather forecasting with machine learning},
    journal = {Nature},
    year    = {2025},
    volume  = {637},
    number  = {8044},
    pages   = {84--90},
    month   = {jan},
    doi     = {10.1038/s41586-024-08252-9},
    note    = {Epub 2024 Dec 4; PMID: 39633054; PMCID: PMC11666454},
}

@article{Appa,
    title={{A}ppa: {B}ending {W}eather {D}ynamics with {L}atent {D}iffusion {M}odels for {G}lobal {D}ata {A}ssimilation},
    author={Gérôme Andry and Sacha Lewin and François Rozet and Omer Rochman and Victor Mangeleer and Matthias Pirlet and Elise Faulx and Marilaure Grégoire and Gilles Louppe},
    booktitle={Machine Learning and the Physical Sciences Workshop (NeurIPS)},
    year={2025},
    url={https://arxiv.org/abs/2504.18720},
}

@article{ERA5,
    author = {Hersbach, Hans and Bell, Bill and Berrisford, Paul and Hirahara, Shoji and Horányi, András and Muñoz-Sabater, Joaquín and Nicolas, Julien and Peubey, Carole and Radu, Raluca and Schepers, Dinand and Simmons, Adrian and Soci, Cornel and Abdalla, Saleh and Abellan, Xavier and Balsamo, Gianpaolo and Bechtold, Peter and Biavati, Gionata and Bidlot, Jean and Bonavita, Massimo and De Chiara, Giovanna and Dahlgren, Per and Dee, Dick and Diamantakis, Michail and Dragani, Rossana and Flemming, Johannes and Forbes, Richard and Fuentes, Manuel and Geer, Alan and Haimberger, Leo and Healy, Sean and Hogan, Robin J. and Hólm, Elías and Janisková, Marta and Keeley, Sarah and Laloyaux, Patrick and Lopez, Philippe and Lupu, Cristina and Radnoti, Gabor and de Rosnay, Patricia and Rozum, Iryna and Vamborg, Freja and Villaume, Sebastien and Thépaut, Jean-Noël},
    title = {The {ERA5} global reanalysis},
    journal = {Quarterly Journal of the Royal Meteorological Society},
    volume = {146},
    number = {730},
    pages = {1999-2049},
    doi = {https://doi.org/10.1002/qj.3803},
    url = {https://rmets.onlinelibrary.wiley.com/doi/abs/10.1002/qj.3803},
    eprint = {https://rmets.onlinelibrary.wiley.com/doi/pdf/10.1002/qj.3803},
    year = {2020}
}

@article{4D_VAR,
    author = {Le Dimet, François-Xavier and Talagrand, Olivier},
    title = {Variational algorithms for analysis and assimilation of meteorological observations: theoretical aspects},
    journal = {Tellus A},
    volume = {38A},
    number = {2},
    pages = {97-110},
    doi = {https://doi.org/10.1111/j.1600-0870.1986.tb00459.x},
    url = {https://onlinelibrary.wiley.com/doi/abs/10.1111/j.1600-0870.1986.tb00459.x},
    eprint = {https://onlinelibrary.wiley.com/doi/pdf/10.1111/j.1600-0870.1986.tb00459.x},
    year = {1986}
}

@incollection{pf_degeneracy,
    volume = {8964},
    editor = {Sai Ravela and Adrian Sandu},
    doi = {10.1007/978-3-319-25138-7},
    year = {2015},
    address = {Heidelberg},
    series = {Lecture notes in computer science},
    publisher = {Springer},
    title = {Aspects of particle filtering in high-dimensional spaces},
    pages = {251--262},
    booktitle = {Dynamic data-driven environmental systems science},
    isbn = {9783319251370},
    url = {https://centaur.reading.ac.uk/50238/},
    issn = {0302-9743},
    author = {Van Leeuwen, Peter}
}

@article {OptimalProposal,
    author = "Chris Snyder and Thomas Bengtsson and Mathias Morzfeld",
    title = "{P}erformance {B}ounds for {P}article {F}ilters {U}sing the {O}ptimal {P}roposal",
    journal = "Monthly Weather Review",
    year = "2015",
    publisher = "American Meteorological Society",
    address = "Boston MA, USA",
    volume = "143",
    number = "11",
    doi = "10.1175/MWR-D-15-0144.1",
    pages=      "4750 - 4761",
    url = "https://journals.ametsoc.org/view/journals/mwre/143/11/mwr-d-15-0144.1.xml"
}

@inproceedings{FlowDAS,
    title={Flow{DAS}: A {S}tochastic {I}nterpolant-based {F}ramework for {D}ata {A}ssimilation},
    author={Siyi Chen and Yixuan Jia and Qing Qu and He Sun and Jeffrey A Fessler},
    booktitle={The Thirty-ninth Annual Conference on Neural Information Processing Systems},
    year={2026},
    url={https://openreview.net/forum?id=1nWqhiulqD}
}

@article{LETKF,
    title = {Efficient data assimilation for spatiotemporal chaos: A local ensemble transform {K}alman filter},
    journal = {Physica D: Nonlinear Phenomena},
    volume = {230},
    number = {1},
    pages = {112-126},
    year = {2007},
    note = {Data Assimilation},
    issn = {0167-2789},
    doi = {https://doi.org/10.1016/j.physd.2006.11.008},
    url = {https://www.sciencedirect.com/science/article/pii/S0167278906004647},
    author = {Brian R. Hunt and Eric J. Kostelich and Istvan Szunyogh}
}

@misc{DAISI,
    title={{DAISI}: {D}ata {A}ssimilation with {I}nverse {S}ampling using {S}tochastic {I}nterpolants}, 
    author={Martin Andrae and Erik Larsson and So Takao and Tomas Landelius and Fredrik Lindsten},
    year={2025},
    eprint={2512.00252},
    archivePrefix={arXiv},
    primaryClass={stat.ML},
    url={https://arxiv.org/abs/2512.00252}, 
}

@inproceedings{MMPS,
    title={{L}earning {D}iffusion {P}riors from {O}bservations by {E}xpectation {M}aximization},
    author={Fran{\c{c}}ois Rozet and G{\'e}r{\^o}me Andry and Francois Lanusse and Gilles Louppe},
    booktitle={The Thirty-eighth Annual Conference on Neural Information Processing Systems},
    year={2024},
    url={https://openreview.net/forum?id=7v88Fh6iSM}
}

@book{Proba,
    address = {New York [u.a.]},
    author = {Billingsley, {Patrick}},
    edition = {3. ed},
    isbn = {0471007102},
    pagetotal = {XII, 593},
    ppn_gvk = {164761632},
    publisher = {Wiley},
    series = {A Wiley-Interscience publication},
    title = {Probability and {M}easure},
    url = {http://gso.gbv.de/DB=2.1/CMD?ACT=SRCHA&SRT=YOP&IKT=1016&TRM=ppn+164761632&sourceid=fbw_bibsonomy},
    year = 1995
}

@inproceedings{DDPM,
    author = {Ho, Jonathan and Jain, Ajay and Abbeel, Pieter},
    booktitle = {Advances in Neural Information Processing Systems},
    editor = {Larochelle, H. and Ranzato, M. and Hadsell, R. and Balcan, M.F. and Lin, H.},
    pages = {6840--6851},
    publisher = {Curran Associates, Inc.},
    title = {{D}enoising {D}iffusion {P}robabilistic {M}odels},
    volume = 33,
    year = 2020
}

@article{SDE,
    title = {Reverse-time diffusion equation models},
    journal = {Stochastic Processes and their Applications},
    volume = {12},
    number = {3},
    pages = {313-326},
    year = {1982},
    issn = {0304-4149},
    doi = {https://doi.org/10.1016/0304-4149(82)90051-5},
    url = {https://www.sciencedirect.com/science/article/pii/0304414982900515},
    author = {Brian D.O. Anderson}
}

@inproceedings{Exponential_solver,
    title={{F}ast {S}ampling of {D}iffusion {M}odels with {E}xponential {I}ntegrator},
    author={Qinsheng Zhang and Yongxin Chen},
    booktitle={The Eleventh International Conference on Learning Representations },
    year={2023},
    url={https://openreview.net/forum?id=Loek7hfb46P}
}

@inproceedings{DDIM,
    title={{D}enoising {D}iffusion {I}mplicit {M}odels},
    author={Jiaming Song and Chenlin Meng and Stefano Ermon},
    booktitle={International Conference on Learning Representations},
    year={2021},
    url={https://openreview.net/forum?id=St1giarCHLP}
}

@inproceedings{Filtering_GenCast,
  title={{T}raining-{F}ree {D}ata {A}ssimilation with {G}en{C}ast},
  author={Savary, Thomas and Rozet, François and Louppe, Gilles},
  booktitle={NeurIPS 2025 Workshop on Tackling Climate Change with Machine Learning},
  url={https://www.climatechange.ai/papers/neurips2025/39},
  year={2025}
}

@article{GMRES,
    address = {Philadelphia, PA, USA},
    author = {Saad, Youcef and Schultz, Martin H.},
    doi = {10.1137/0907058},
    issn = {0196-5204},
    journal = {SIAM Journal on Scientific and Statistical Computing},
    month = jul,
    number = 3,
    pages = {856--869},
    publisher = {Society for Industrial and Applied Mathematics},
    title = {{GMRES: A Generalized Minimal Residual Algorithm for Solving Nonsymmetric Linear Systems}},
    url = {http://dx.doi.org/10.1137/0907058},
    volume = 7,
    year = 1986
}

@article{BiCG,
    author = {van der Vorst, H. A.},
    title = {{B}i-{CGSTAB}: {A} {F}ast and {S}moothly {C}onverging {V}ariant of {B}i-{CG} for the {S}olution of {N}onsymmetric {L}inear {S}ystems},
    journal = {SIAM Journal on Scientific and Statistical Computing},
    volume = {13},
    number = {2},
    pages = {631-644},
    year = {1992},
    doi = {10.1137/0913035},
    URL = {https://doi.org/10.1137/0913035},
    eprint = {https://doi.org/10.1137/0913035}
}

@article{L63_Stochastic,
    title = {Stochastic climate dynamics: Random attractors and time-dependent invariant measures},
    journal = {Physica D: Nonlinear Phenomena},
    volume = {240},
    number = {21},
    pages = {1685-1700},
    year = {2011},
    issn = {0167-2789},
    doi = {https://doi.org/10.1016/j.physd.2011.06.005},
    url = {https://www.sciencedirect.com/science/article/pii/S016727891100145X},
    author = {Mickaël D. Chekroun and Eric Simonnet and Michael Ghil}
}

@article{Milstein,
    author = {Gelbrich, Matthias and R\"{o}misch, Werner},
    title = {{N}umerical {S}olution of {S}tochastic {D}ifferential {E}quations ({P}eter {E}. {K}loeden and {E}ckhard {P}laten)},
    journal = {SIAM Review},
    volume = {37},
    number = {2},
    pages = {272-275},
    year = {1995},
    doi = {10.1137/1037073},
    URL = {https://doi.org/10.1137/1037073},
    eprint = {https://doi.org/10.1137/1037073},
}

@inproceedings{resnet,
    author = {He, Kaiming and Zhang, Xiangyu and Ren, Shaoqing and Sun, Jian},
    booktitle = {Proceedings of 2016 IEEE Conference on Computer Vision and Pattern Recognition},
    doi = {10.1109/CVPR.2016.90},
    issn = {1063-6919},
    location = {Las Vegas, NV, USA},
    month = jun,
    pages = {770--778},
    publisher = {IEEE},
    series = {CVPR '16},
    title = {{Deep Residual Learning for Image Recognition}},
    url = {http://ieeexplore.ieee.org/document/7780459},
    year = 2016
}

@inproceedings{EDM,
    title={{E}lucidating the {D}esign {S}pace of {D}iffusion-{B}ased {G}enerative {M}odels},
    author={Tero Karras and Miika Aittala and Timo Aila and Samuli Laine},
    booktitle={Advances in Neural Information Processing Systems},
    editor={Alice H. Oh and Alekh Agarwal and Danielle Belgrave and Kyunghyun Cho},
    year={2022},
    url={https://openreview.net/forum?id=k7FuTOWMOc7}
}

@article{num_ens_members,
    author = {Leutbecher, Martin},
    title = {Ensemble size: How suboptimal is less than infinity?},
    journal = {Quarterly Journal of the Royal Meteorological Society},
    volume = {145},
    number = {S1},
    pages = {107-128},
    doi = {https://doi.org/10.1002/qj.3387},
    url = {https://rmets.onlinelibrary.wiley.com/doi/abs/10.1002/qj.3387},
    eprint = {https://rmets.onlinelibrary.wiley.com/doi/pdf/10.1002/qj.3387},
    year = {2019}
}

@misc{EnFF,
    title={{F}low {M}atching for {E}fficient and {S}calable {D}ata {A}ssimilation}, 
    author={Taos Transue and Bohan Chen and So Takao and Bao Wang},
    year={2025},
    eprint={2508.13313},
    archivePrefix={arXiv},
    primaryClass={stat.ML},
    url={https://arxiv.org/abs/2508.13313}, 
}

@inproceedings{Latent_EnSF,
    title={{L}atent-{E}n{SF}: {A} {L}atent {E}nsemble {S}core {F}ilter for {H}igh-{D}imensional {D}ata {A}ssimilation with {S}parse {O}bservation {D}ata},
    author={Phillip Si and Peng Chen},
    booktitle={The Thirteenth International Conference on Learning Representations},
    year={2025},
    url={https://openreview.net/forum?id=urcEYsZOBz}
}

@InProceedings{U-Net,
author="Ronneberger, Olaf
and Fischer, Philipp
and Brox, Thomas",
editor="Navab, Nassir
and Hornegger, Joachim
and Wells, William M.
and Frangi, Alejandro F.",
title="{U}-{N}et: {C}onvolutional {N}etworks for {B}iomedical {I}mage {S}egmentation",
booktitle="Medical Image Computing and Computer-Assisted Intervention -- MICCAI 2015",
year="2015",
publisher="Springer International Publishing",
address="Cham",
pages="234--241",
isbn="978-3-319-24574-4"
}

@article {Optimal_Proposal_Efficiency,
    author = "Laura Slivinski and Chris Snyder",
    title = "{E}xploring {P}ractical {E}stimates of the {E}nsemble {S}ize {N}ecessary for {P}article {F}ilters",
    journal = "Monthly Weather Review",
    year = "2016",
    publisher = "American Meteorological Society",
    address = "Boston MA, USA",
    volume = "144",
    number = "3",
    doi = "10.1175/MWR-D-14-00303.1",
    pages=      "861 - 875",
    url = "https://journals.ametsoc.org/view/journals/mwre/144/3/mwr-d-14-00303.1.xml"
}

@inproceedings{DiffusionLAM,
    title={{D}iffusion-{LAM}: {P}robabilistic {L}imited {A}rea {W}eather {F}orecasting with {D}iffusion},
    author={Larsson, Erik and Oskarsson, Joel and Landelius, Tomas and Lindsten, Fredrik},
    booktitle={ICLR 2025 Workshop on Tackling Climate Change with Machine Learning},
    url={https://www.climatechange.ai/papers/iclr2025/36},
    year={2025}
}

@article{Diffusion_Sea_Ice,
    author = {Finn, Tobias Sebastian and Durand, Charlotte and Farchi, Alban and Bocquet, Marc and Rampal, Pierre and Carrassi, Alberto},
    title = {{G}enerative {D}iffusion for {R}egional {S}urrogate {M}odels From {S}ea-{I}ce {S}imulations},
    journal = {Journal of Advances in Modeling Earth Systems},
    volume = {16},
    number = {10},
    pages = {e2024MS004395},
    doi = {https://doi.org/10.1029/2024MS004395},
    url = {https://agupubs.onlinelibrary.wiley.com/doi/abs/10.1029/2024MS004395},
    eprint = {https://agupubs.onlinelibrary.wiley.com/doi/pdf/10.1029/2024MS004395},
    note = {e2024MS004395 2024MS004395},
    year = {2024}
}

@article{IFS,
    title = {The {IFS} model: A parallel production weather code},
    journal = {Parallel Computing},
    volume = {21},
    number = {10},
    pages = {1621-1638},
    year = {1995},
    note = {Climate and weather modeling},
    issn = {0167-8191},
    doi = {https://doi.org/10.1016/0167-8191(96)80002-0},
    url = {https://www.sciencedirect.com/science/article/pii/0167819196800020},
    author = {S.R.M. Barros and D. Dent and L. Isaksen and G. Robinson and G. Mozdzynski and F. Wollenweber}
}

@article{Tweedie,
    ISSN = {01621459},
    URL = {http://www.jstor.org/stable/23239562},
    author = {Bradley Efron},
    journal = {Journal of the American Statistical Association},
    number = {496},
    pages = {1602--1614},
    publisher = {[American Statistical Association, Taylor & Francis, Ltd.]},
    title = {{T}weedie's {F}ormula and {S}election {B}ias},
    urldate = {2026-01-28},
    volume = {106},
    year = {2011}
}

@article{SSR,
    author="V.  Fortin and M.  Abaza and F.  Anctil and R.  Turcotte",
    title="{W}hy {S}hould {E}nsemble {S}pread {M}atch the {RMSE} of the {E}nsemble {M}ean?",
    journal="Journal of Hydrometeorology",
    year="2014",
    publisher="American Meteorological Society",
    address="Boston MA, USA",
    volume="15",
    number="4",
    doi="10.1175/JHM-D-14-0008.1",
    pages="1708 - 1713",
    url="https://journals.ametsoc.org/view/journals/hydr/15/4/jhm-d-14-0008_1.xml"
}

@article{Smoothing_Doucet,
    author={Doucet, Arnaud and Johansen, Adam},
    year={2009},
    month={01},
    pages={},
    title={{A} {T}utorial on {P}article {F}iltering and {S}moothing: {F}ifteen {Y}ears {L}ater},
    volume={12},
    journal={Handbook of Nonlinear Filtering}
}

@inproceedings{Particle_Smoothing,
    title="Fast particle smoothing: if {I} had a million particles",
    author="Klaas, Mike and Briers, Mark and de Freitas, Nando and Doucet, Arnaud and Maskell, Simon and Lang, Dustin",
    year="2006",
    address="New York, NY, USA",
    booktitle="International Conference on Machine Learning (ICML)",
    location="Pittsburgh, Pennsylvania",
    pages="481--488",
    publisher="ACM",
    url="http://doi.acm.org/10.1145/1143844.1143905",
    doi="10.1145/1143844.1143905",
}

@article{CRPS,
    author = {Tilmann Gneiting and Adrian E Raftery},
    title = {{S}trictly {P}roper {S}coring {R}ules, {P}rediction, and {E}stimation},
    journal = {Journal of the American Statistical Association},
    volume = {102},
    number = {477},
    pages = {359--378},
    year = {2007},
    publisher = {Taylor \& Francis},
    doi = {10.1198/016214506000001437},
    URL = {https://doi.org/10.1198/016214506000001437},
    eprint = {https://doi.org/10.1198/016214506000001437}
}

@inproceedings{DPS,
    title={{D}iffusion {P}osterior {S}ampling for {G}eneral {N}oisy {I}nverse {P}roblems},
    author={Hyungjin Chung and Jeongsol Kim and Michael Thompson Mccann and Marc Louis Klasky and Jong Chul Ye},
    booktitle={The Eleventh International Conference on Learning Representations },
    year={2023},
    url={https://openreview.net/forum?id=OnD9zGAGT0k},
}

@article{EnKG,
    title={{E}nsemble {K}alman {D}iffusion {G}uidance: {A} {D}erivative-free {M}ethod for {I}nverse {P}roblems},
    author={Hongkai Zheng and Wenda Chu and Austin Wang and Nikola Borislavov Kovachki and Ricardo Baptista and Yisong Yue},
    journal={Transactions on Machine Learning Research},
    issn={2835-8856},
    year={2025},
    url={https://openreview.net/forum?id=XPEEsKneKs},
    note={}
}

@InProceedings{TARP,
    title={{S}ampling-{B}ased {A}ccuracy {T}esting of {P}osterior {E}stimators for {G}eneral {I}nference},
    author={Lemos, Pablo and Coogan, Adam and Hezaveh, Yashar and Perreault-Levasseur, Laurence},
    booktitle={Proceedings of the 40th International Conference on Machine Learning},
    pages={19256--19273},
    year={2023},
    editor={Krause, Andreas and Brunskill, Emma and Cho, Kyunghyun and Engelhardt, Barbara and Sabato, Sivan and Scarlett, Jonathan},
    volume={202},
    series={Proceedings of Machine Learning Research},
    month={23--29 Jul},
    publisher={PMLR},
    url={https://proceedings.mlr.press/v202/lemos23a.html}
}

@misc{MIRA,
    title={{MIRA}: {A} {S}core for {C}onditional {D}istribution {A}ccuracy and {M}odel {C}omparison}, 
    author={Sammy Sharief and Justine Zeghal and Gabriel Missael Barco and Pablo Lemos and Yashar Hezaveh and Laurence Perreault-Levasseur},
    year={2026},
    eprint={2605.02014},
    archivePrefix={arXiv},
    primaryClass={stat.ML},
    url={https://arxiv.org/abs/2605.02014}, 
}
\bibliographystyle{icml2026}

\newpage
\appendix
\onecolumn

\section{Tweedie’s formulas} \label{appendix:Tweedie}
\begin{theorem}
Assuming that $p_{t}(x_{t}\mid x) = \mathcal{N}(x_{t} \mid \alpha_{t} x, \Sigma_{t})$, the first and second moments of $p_{t}(x \mid x_{t})$ are linked to the score function $\nabla_{\!x_{t}} \log p_{t}(x_{t}) $ used in Equation~\eqref{eq:reverse_diffusion_equation} through
\begin{align}
    &\mathbb{E}[x \mid x_{t}] = \alpha_{t}^{-1} \left [x_{t} + \Sigma_{t} \nabla_{\!x_{t}} \log p_{t}(x_{t})\right], \\
    &\mathbb{V}[x \mid x_{t}] = \alpha_{t}^{-2} \left [\Sigma_{t} + \Sigma_{t} \nabla_{\!x_{t}}^{2} \log p_{t}(x_{t}) \right].
\end{align}
\end{theorem}
We provide proofs of this theorem for completeness, even though it is a well known result \cite{Tweedie}.
\begin{proof}
    \setlength{\jot}{10pt}
    \begin{align*}
        \nabla_{\!x_{t}} \log p_{t}(x_{t}) &= \frac{1}{p_{t}(x_{t})} \int \nabla_{\!x_{t}} p_{t}(x_{t}, x) \mathrm{d}x \\
        &= \frac{1}{p_{t}(x_{t})} \int p_{t}(x_{t}, x) \nabla_{\!x_{t}} \log p_{t}(x_{t}, x) \mathrm{d}x \\
        &= \int p_{t}(x \mid x_{t}) \nabla_{\!x_{t}} \log p_{t}(x_{t} \mid x) \text{d}x \\
        &= \int p_{t}(x \mid x_{t}) \Sigma_{t}^{-1} \left( \alpha_{t} x - x_{t} \right) \mathrm{d}x \\
        &= \Sigma_{t}^{-1}\left( \alpha_{t} \mathbb{E}[x \mid x_{t}]  - x_{t} \right)
    \end{align*}
\end{proof}

\begin{proof}
    \setlength{\jot}{10pt}
    \begin{align*}
        \nabla_{\!x_{t}}^{2} \log p_{t}(x_{t}) &= \nabla_{\!x_{t}} \left [ \nabla_{\!x_{t}}^{\top} \log p_{t}(x_{t})\right] \\
        &= \alpha_{t} \nabla_{\!x_{t}}^{\top} \mathbb{E}[x \mid x_{t}]\Sigma_{t}^{-1} - \Sigma_{t}^{-1}\\
        &= \alpha_{t} \left [ \int \nabla_{\!x_{t}}p_{t}(x \mid x_{t})x^{\top} \mathrm{d}x \right] \Sigma_{t}^{-1} - \Sigma_{t}^{-1}\\
        &= \alpha_{t} \left [ \int p_{t}(x \mid x_{t}) \nabla_{\!x_{t}} \log p_{t}(x \mid x_{t})x^{\top} \mathrm{d}x \right] \Sigma_{t}^{-1} - \Sigma_{t}^{-1}\\
        &= \alpha_{t} \left [ \int p_{t}(x \mid x_{t}) \alpha_{t} \Sigma_{t}^{-1}(x - \mathbb{E}[x \mid x_{t}])x^{\top} \mathrm{d}x \right] \Sigma_{t}^{-1} - \Sigma_{t}^{-1}\\
        &= \alpha_{t}^{2} \Sigma_{t}^{-1} \left( \mathbb{E}[xx^{\top} \mid x_{t}] - \mathbb{E}[x \mid x_{t}] \mathbb{E}[x \mid x_{t}]^{\top}\right)\Sigma_{t}^{-1} - \Sigma_{t}^{-1} \\
        &= \alpha_{t}^{2} \Sigma_{t}^{-1} \mathbb{V}[x \mid x_{t}]\Sigma_{t}^{-1} - \Sigma_{t}^{-1} \\
    \end{align*}
\end{proof}

\newpage
\section{Training-free ensemble methods} \label{appendix:DA_algo}
In the Lorenz'63 experiment, we compared our method with four other training-free ensemble algorithms. These methods approximate the Bayesian filtering distribution at each assimilation step $k$ using an ensemble of $N$ particles, relying only on the transition distribution $p(x^{k+1} \mid x^{k})$ and the observation operator $\mathcal{H}$.

\subsection{Bootstrap Particle Filter (BPF)\cite{PF}}
The Bootstrap Particle Filter (BPF, Algorithm \ref{algo:BPF}) is the simplest particle filter that samples new particles from the transition distribution $p(x^{k+1} \mid x^{k})$. In theory, it converges to the Bayesian filtering distribution, even with nonlinear dynamics and observation operators, but requires a large number of particles to perform well.

\begin{algorithm}    
    \caption{Bootstrap Particle Filter (BPF)}
    \begin{algorithmic}[1]
        \STATE {\bfseries Inputs:} $p(x^{0})$, $N$, $K$, $N_{\mathrm{thr}}$
        \STATE $x^{0}_{i} \sim p(x^{0})$
        \STATE $w^{0}_{i} \gets 1/N$
        \FOR{$k=0$ {\bfseries to} $K-1$}
            \STATE $x^{k+1}_{i} \sim p(x^{k+1} \mid x^{k}_{i})$
            \STATE $\hat{w}^{k+1}_{i} \gets p(y^{k+1} \mid x^{k+1}_{i})$
            \STATE $w^{k+1}_{i} \gets \hat{w}^{k+1}_{i} / \sum_{j=1}^{N} \hat{w}^{k+1}_{j}$
            \STATE $N_{\text{eff}} \gets 1 / \sum_{i=1}^{N}(w^{k+1}_{i})^{2}$
            \IF{$N_{\mathrm{eff}} < N_{\mathrm{thr}}$}
                \STATE do resampling
            \ENDIF
        \ENDFOR
        \STATE {\bfseries Return} $\mu^{k}_{x} =\sum_{i = 1}^{N} w^{k}_{i} \delta_{x^{k}_{i}}$ for all $k \in [1,K]$
    \end{algorithmic}
    \label{algo:BPF}
\end{algorithm}

\subsection{Ensemble Kalman Filter (EnKF) \cite{EnKF_Evensen}}
The Ensemble Kalman Filter (EnKF, Algorithm \ref{algo:EnKF}) is a popular ensemble-based method for Bayesian filtering. Although it assumes near-Gaussian distributions, EnKF performs well in practice and its variants are widely used in high-dimensional applications.

\begin{algorithm}    
    \caption{Ensemble Kalman Filter (EnKF)}
    \begin{algorithmic}[1]
        \STATE {\bfseries Inputs:} $p(x^{0})$, $N$, $K$
        \STATE $x^{0}_{i} \sim p(x^{0})$
        \FOR{$k=0$ {\bfseries to} $K-1$}
            \STATE $x^{f}_{i} \sim p(x^{k+1} \mid x^{k}_{i})$
            \STATE $h^{f}_{i} \gets \mathcal{H}(x^{f}_{i})$
            \STATE $\bar{x}^{f} \gets \frac{1}{N} \sum_{i=1}^{N} x^{f}_{i}$
            \STATE $\bar{h}^{f} \gets \frac{1}{N} \sum_{i=1}^{N} h^{f}_{i}$
            \STATE $P_{yy} \gets \frac{1}{N-1} \sum_{i=1}^{N} (h^{f}_{i} - \bar{h}^{f})(h^{f}_{i} - \bar{h}^{f})^T$
            \STATE $P_{xy} \gets \frac{1}{N-1} \sum_{i=1}^{N} (x^{f}_{i} - \bar{x}^{f})(h^{f}_{i} - \bar{h}^{f})^T$
            \STATE $G \gets P_{xy} (P_{yy} + \Sigma_{y})^{-1}$
            \FOR{$i=1$ {\bfseries to} $N$}
                \STATE $\epsilon_{i} \sim \mathcal{N}(0, \Sigma_{y})$
                \STATE $d_{i} \gets y^{k+1} + \epsilon_{i}$
                \STATE $x^{k+1}_{i} \gets x^{f}_{i} + G (d_{i} - h^{f}_{i})$
            \ENDFOR
        \ENDFOR
        \STATE {\bfseries Return} $\mu^{k}_{x} =\frac{1}{N} \sum_{i = 1}^{N}\delta_{x^{k}_{i}}$ for all $k \in [1,K]$
    \end{algorithmic}
    \label{algo:EnKF}
\end{algorithm}

\subsection{Ensemble Score Filter (EnSF) \cite{EnSF}}
The Ensemble Score Filter (EnSF, Algorithm \ref{algo:EnSF}) is a recent Bayesian filtering algorithm inspired by diffusion models. The method consists in estimating the score by Monte Carlo and then updating it using the gradient of the likelihood.

\begin{algorithm}    
    \caption{Ensemble Score Filter (EnSF)}
    \begin{algorithmic}[1]
        \STATE {\bfseries Inputs:} $p(x^{0})$, $N$, $\Delta_{t}$, $K$, $\alpha(\cdot)$, $\beta(\cdot)$, $b(\cdot)$, $\sigma(\cdot)$
        \STATE $x^{0}_{i} \sim p(x^{0})$
        \FOR{$k=0$ {\bfseries to} $K-1$}
            \STATE $x^{f}_{i} \sim p(x^{k+1} \mid x^{k}_{i})$
            \FOR{$i=1$ {\bfseries to} $N$}
                \STATE $z_{i} \sim \mathcal{N}(0, \beta(1)^{2}I)$
                \FOR{$t$ in $[1, 1 - \Delta_{t}, \cdots, \Delta_{t}]$}
                    \STATE $w_{j} \gets \mathcal{N}(z_{i} \mid \alpha(t)x^{f}_{j}, \beta_{t}^{2}I) / \sum_{k=1}^{N} \left [ \mathcal{N}(z_{i} \mid \alpha(t)x^{f}_{k}, \beta_{t}^{2}I) \right]$
                    \STATE $s_{x} \gets - \sum_{j=1}^{N}\left( \frac{z_{i} - \alpha(t)x^{f}_{j}}{\beta(t)^{2}}\right) w_{j}$
                    \STATE  $s_{x,y} \gets s_{x} + (1-t) \times \nabla_{\!z_{i}} \log p(y^{k+1} \mid z_{i})$
                    \STATE $z_{i} \gets z_{i} - \left [ b(t)z_{i} - \sigma(t)^{2}s_{x,y}\right]\Delta_{t} - \sigma(t)\sqrt{\Delta_{t}} \varepsilon,~\varepsilon \sim \mathcal{N}(0,I)$
                \ENDFOR
                \STATE $x^{k+1}_{i} \gets z_{i}$
            \ENDFOR
        \ENDFOR
        \STATE {\bfseries Return} $\mu^{k}_{x} = \frac{1}{N}\sum_{i = 1}^{N}\delta_{x^{k}_{i}}$ for all $k \in [1,K]$
    \end{algorithmic}
    \label{algo:EnSF}
\end{algorithm}

\subsection{Ensemble Flow Filter (EnFF) \cite{EnFF}}
The Ensemble Flow Filter (EnFF, Algorithm \ref{algo:EnFF}) is a flow-matching-based variant of the EnSF algorithm presented above. The idea is to take advantage of the flow matching framework to construct straighter paths between two successive filtering distributions, leading to more efficient sampling.
\begin{algorithm}    
    \caption{Ensemble Flow Filter (EnFF)}
    \begin{algorithmic}[1]
        \STATE {\bfseries Inputs:} $p(x^{0})$, $N$, $K$, $\Delta_{t}$, $\lambda$, $\sigma_{\mathrm{min}}$ 
        \STATE $x^{0}_{i} \sim p(x^{0})$
        \FOR{$k=0$ {\bfseries to} $K-1$}
            \STATE $x^{f}_{i} \sim p(x^{k+1} \mid x^{k}_{i})$
            \FOR{$i=1$ {\bfseries to} $N$}
                \STATE $z_{i} \gets x^{k}_{i}$
                \FOR{$t$ in $[0, \Delta_{t}, \cdots, 1 -\Delta_{t}]$}
                    \STATE $w_{j} \gets \mathcal{N}(z_{i} \mid tx^{f}_{j} + (1-t)x^{k}_{j}, \sigma_{\mathrm{min}}^{2}I) / \sum_{k=1}^{N} \left [ \mathcal{N}(z_{i} \mid tx^{f}_{k} + (1-t)x^{k}_{k}, \sigma_{\mathrm{min}}^{2}I) \right]$
                    \STATE $u \gets \sum_{j=1}^{N}\left(x^{f}_{j}  - x^{k}_{j}\right) w_{j}$
                    \STATE $\hat{z}_{i} \gets z_{i} + (1-t) \times u$
                    \STATE  $\tilde{u} \gets u - \lambda \times \nabla_{\!\hat{z}_{i}} \log p(y^{k+1} \mid \hat{z}_{i})$
                    \STATE $z_{i} \gets z_{i} + \Delta_{t} \times \tilde{u}$
                \ENDFOR
                \STATE $x^{k+1}_{i} \gets z_{i}$
            \ENDFOR
        \ENDFOR
        \STATE {\bfseries Return} $\mu^{k}_{x} = \frac{1}{N}\sum_{i = 1}^{N}\delta_{x^{k}_{i}}$ for all $k \in [1,K]$
    \end{algorithmic}
    \label{algo:EnFF}
\end{algorithm}

\newpage 
\section{Metrics} \label{appendix:metrics}
\subsection{Skill}
The skill of an ensemble of $N$ particles $\left \{x_{i}^{k} \right \}_{1 \leq i \leq N}$ at time $k$ is defined as the RMSE of the ensemble mean
\begin{equation}
    \mathrm{Skill} = \sqrt{\left \langle \left( u^{k} - \frac{1}{N}\sum_{i=1}^{N} x_{i}^{k} \right) ^{2}\right \rangle}
\end{equation}
where $\langle \cdot \rangle$ denotes the spatial mean operator and $u^{k}$ the ground truth at step $k$. In Section \ref{subsection: Lorenz63} and \ref{subsection: NS}, we compute the skill for each assimilation steps and then compute the average skill over the experiment.

\subsection{Spread}
The spread of an ensemble of $N$ particles $\left \{x_{i}^{k} \right \}_{1 \leq i \leq N}$ at time $k$ is defined as the ensemble standard deviation
\begin{equation}
    \mathrm{Spread} = \sqrt{\left \langle \frac{1}{N-1}\sum_{i=1}^{N} \left(x_{i}^{k} - \frac{1}{N}\sum_{j=1}^{N} x_{j}^{k}\right)^{2} \right \rangle}.
\end{equation}

\subsection{Spread-to-skill ratio (SSR)}
As shown by \citet{SSR}, a well-calibrated forecast should have a spread-to-skill ratio of 1, which is a necessary but not sufficient condition. Ratios below one indicate overconfident estimates, whereas ratios above one indicate underconfident estimates. In Section \ref{subsection: Lorenz63}, we compute the SSR for each assimilation steps and then compute the average SSR over the experiment.

\subsection{Continuous ranked probability score (CRPS)}
The CRPS score \cite{CRPS} of an assimilation experiment is defined as
\begin{equation}
    \mathrm{CRPS} = \frac{1}{K} \sum_{k=1}^{K} \left(\frac{1}{N} \sum_{i=1}^{N} \lVert u^{k} - x^{k}_{i} \rVert_{L_{1}} - \frac{1}{2N(N-1)}\sum_{i=1}^{N} \sum_{j=1}^{N} \lVert x^{k}_{i} - x^{k}_{j} \rVert_{L_{1}}\right) 
\end{equation}
where $K$ corresponds to the number of assimilation steps, $N$ to the number of particles, $\left \{x_{i}^{k} \right \}_{1 \leq i \leq N}$ to the ensemble of particles at time $k$, and $u^{k}$ to the ground truth at time $k$. The first term penalizes the average divergence from the ground truth while the second term encourages spread. Therefore, the CRPS is lowest when the distribution of the ensemble matches the ground-truth distribution.

\newpage
\section{Additional results} \label{appendix:additional_results}

\subsection{Incompressible Navier-Stokes Flow} \label{appendix:additional_results_NS}
    \begin{figure}[h!]
      \begin{center}
        \centerline{\includegraphics[width=0.8\columnwidth]{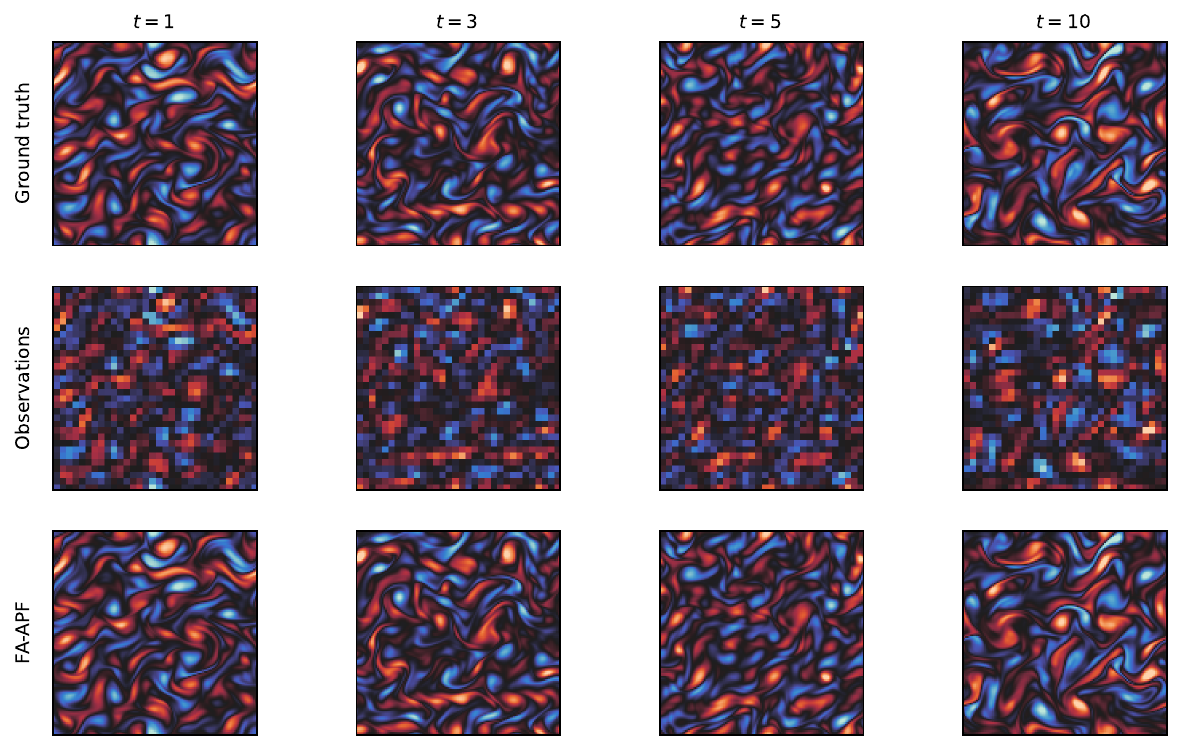}}
        \caption{
          Ground truth, $32 \times 32$ coarse observation, and FA-APF ensemble mean at different time steps during a filtering experiment.
        }
        \label{fig:appendice_ns_32}
      \end{center}
    \end{figure}
    \begin{figure}[b!]
      \begin{center}
        \centerline{\includegraphics[width=0.8\columnwidth]{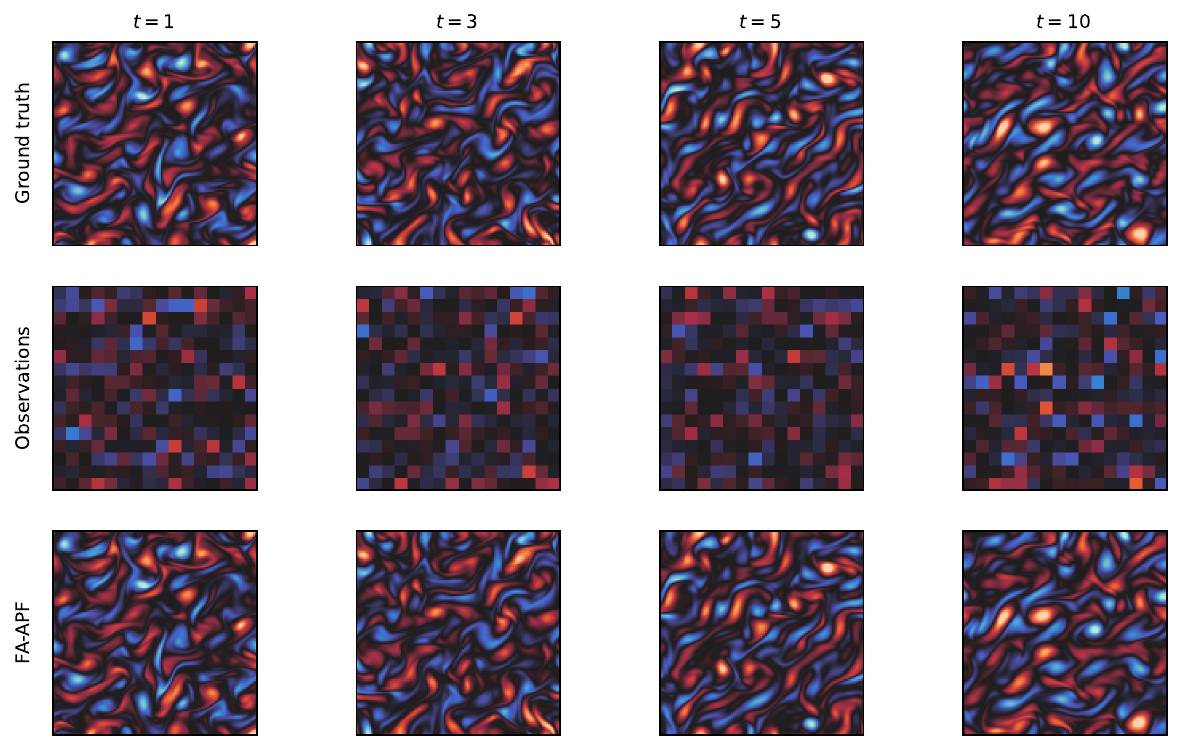}}
        \caption{
          Ground truth, $16 \times 16$ coarse observation, and FA-APF ensemble mean at different time steps during a filtering experiment.
        }
        \label{fig:appendice_ns_16}
      \end{center}
    \end{figure}

\newpage
\subsection{Medium-range weather forecasts (GenCast)} \label{appendix:additional_results_GenCast}
\subsubsection{Metrics for all variables}
    \thispagestyle{empty}
    \vspace*{\fill}
    \begin{center}
        \includegraphics[width=\columnwidth]{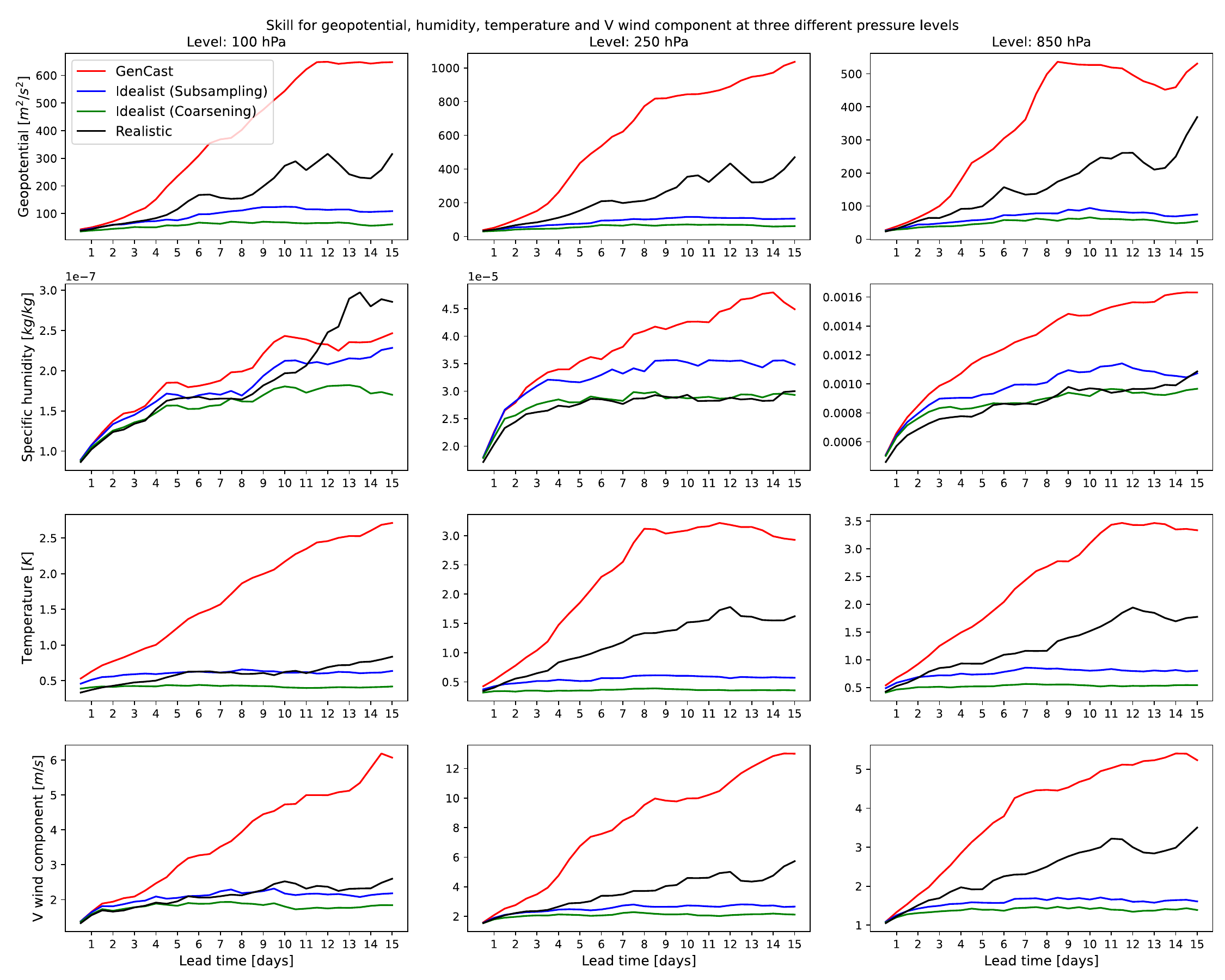}
        \captionof{figure}{Skill for temperature, geopotential, V component of wind and specific humidity at three different pressure levels (100, 250 and 850 hPa). For experiments with sparse temperature observations (blue and green curves), the skill reaches a plateau after a certain number of time steps for all variables (even those that are not observed), well below the one of GenCast’s forecasts.}
        \label{fig:appendice_gencast_all_skill}
    \end{center}
    \vspace*{\fill}
    \clearpage
    
    \newpage
    \thispagestyle{empty}
    \vspace*{\fill}
    \begin{center}
        \includegraphics[width=\columnwidth]{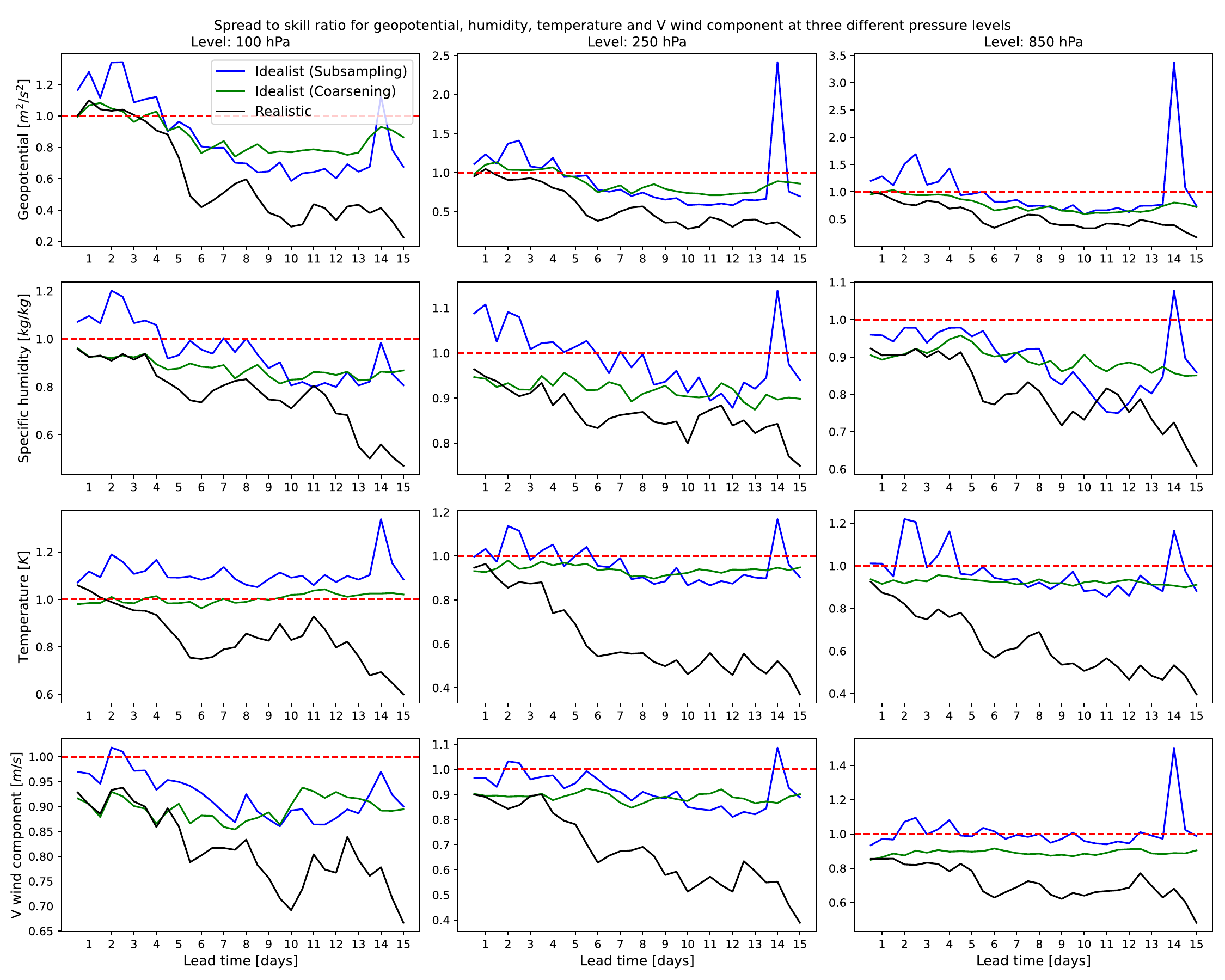}
        \captionof{figure}{Spread-to-skill ratio for temperature, geopotential, V component of wind and specific humidity at three different pressure levels (100, 250 and 850 hPa). For experiments with sparse temperature observations (blue and green curves), the ratio is close to 1, indicating that ensembles are well calibrated.}
        \label{fig:appendice_gencast_all_ssr}
    \end{center}
    \vspace*{\fill}
    \clearpage

\newpage
\subsubsection{Visualization of trajectories}
    \vspace{1cm}
    \begin{figure}[h!]
      \begin{center}
        \centerline{\includegraphics[width=1\columnwidth]{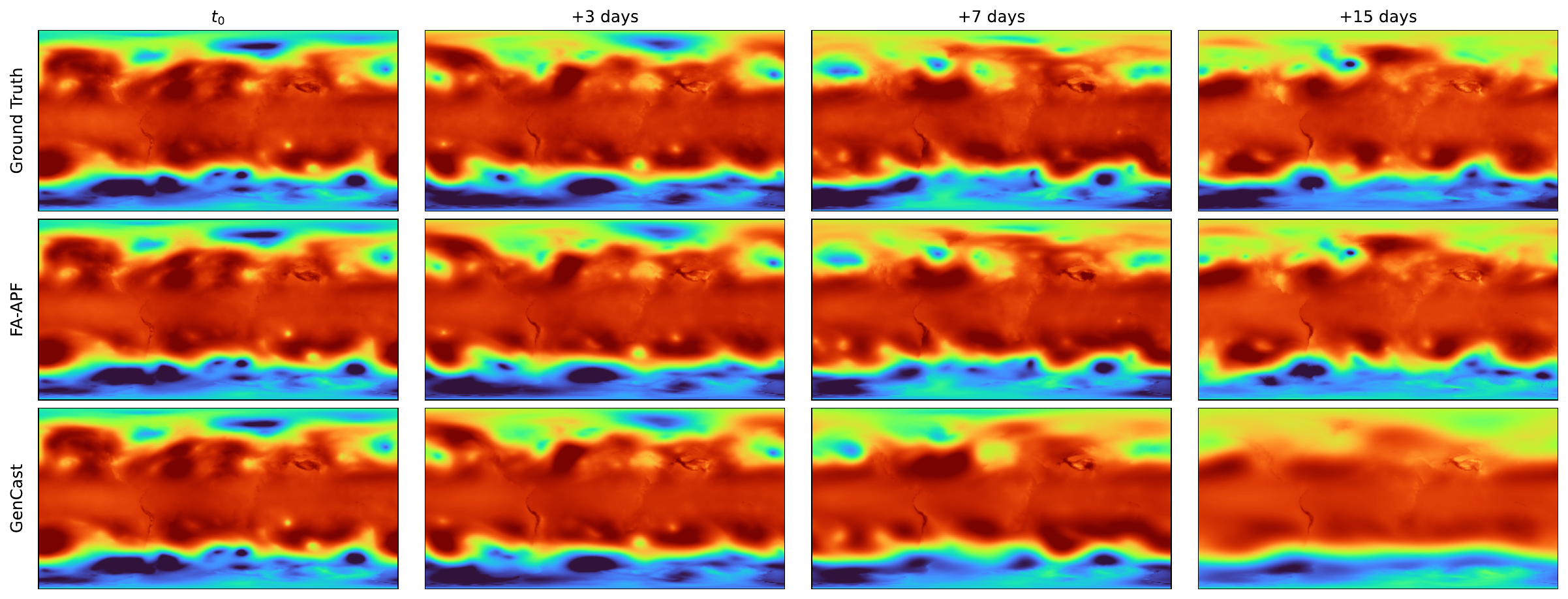}}
        \caption{
          Comparison of the geopotential at 500 hPa between the reference ERA5 trajectory (first row), the FA-APF ensemble mean with realistic observations (second row), and the GenCast ensemble mean (third row) after 3, 7, and 15 days. The ensemble mean of FA-APF remains qualitatively close to the ground truth, even under difficult observation conditions.
        }
        \label{fig:appendice_gencast_geopotential}
      \end{center}
    \end{figure}
    
    \vspace{-1cm}
    \begin{figure}[b!]
      \begin{center}
        \centerline{\includegraphics[width=1\columnwidth]{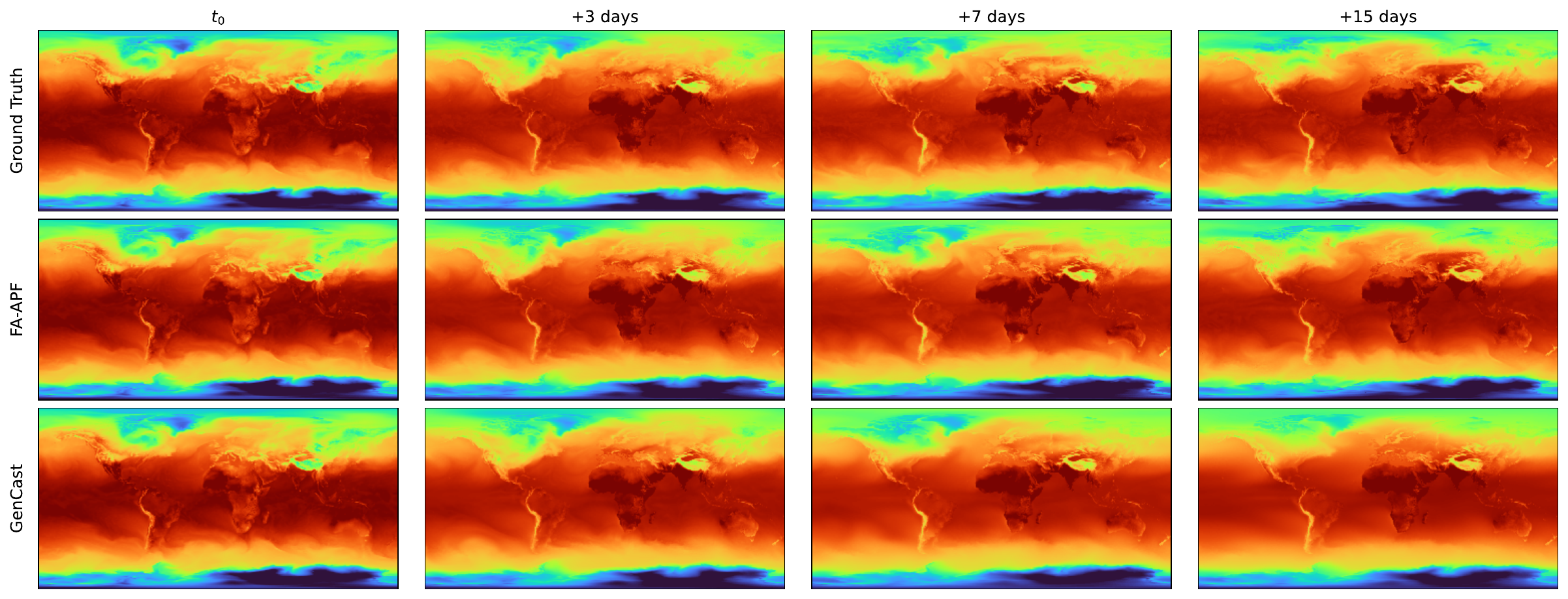}}
        \caption{
          Comparison of surface temperature between the reference ERA5 trajectory (first row), the FA-APF ensemble mean with realistic observations (second row), and the GenCast ensemble mean (third row) after 3, 7, and 15 days. The ensemble mean of FA-APF remains qualitatively close to the ground truth, even under difficult observation conditions.
        }
        \label{fig:appendice_gencast_t2m}
      \end{center}
    \end{figure}

\newpage
\subsubsection{Posterior Predictive Check (PPC)}
    \begin{figure}[h!]
      \begin{center}
        \centerline{\includegraphics[width=0.75\columnwidth]{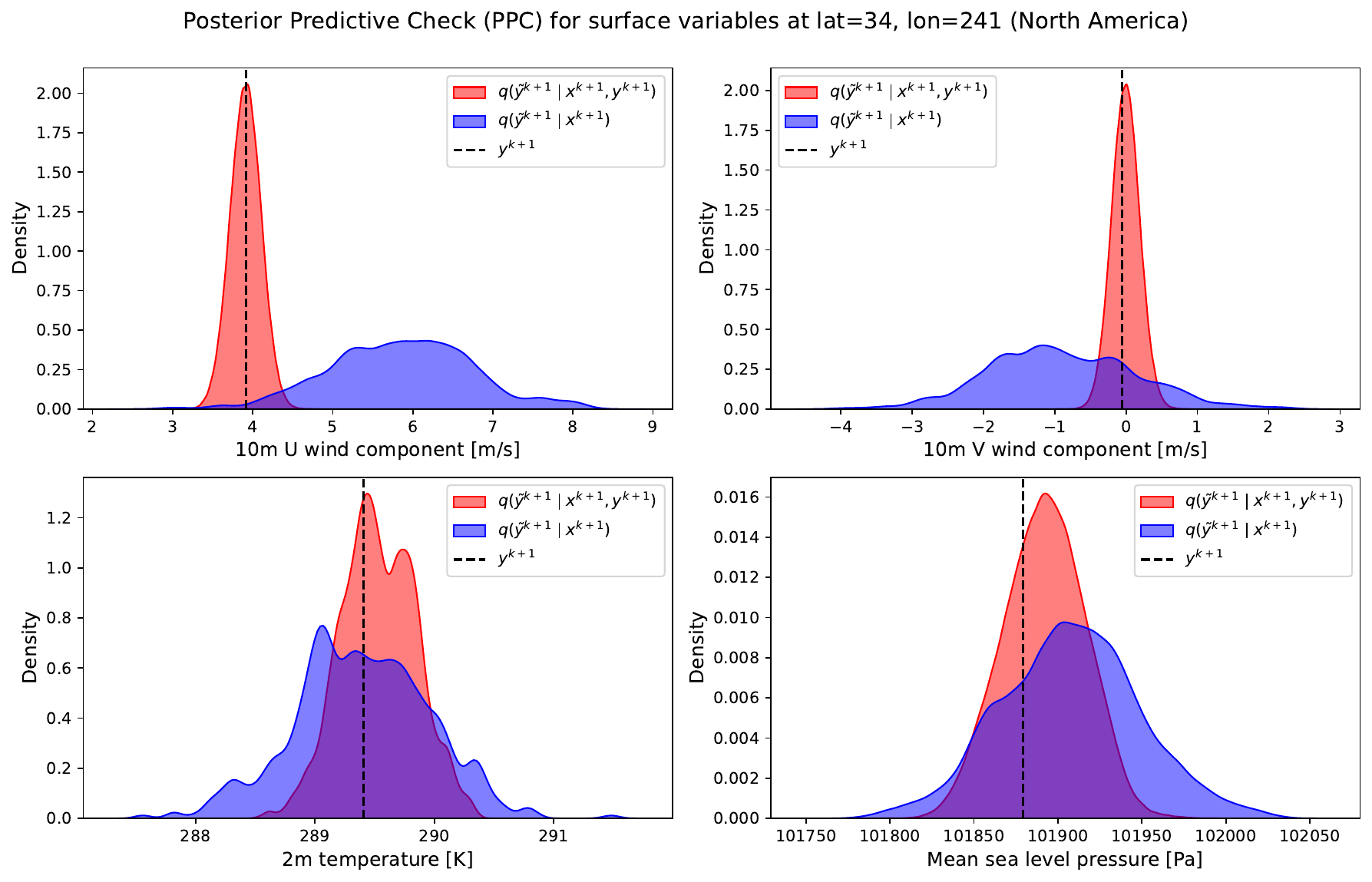}}
        \caption{
          Comparison between the distributions of conditional samples (red curve, generated using the optimal proposal) and unconditional samples (blue curve, generated with GenCast without conditioning) at an observed grid point (in North America). Observations (black dotted lines) are more likely in the distribution of conditional samples.
        }
        \label{fig:appendice_gencast_ppc_america}
      \end{center}
    \end{figure}
    
    \begin{figure}[b!]
      \begin{center}
        \centerline{\includegraphics[width=0.75\columnwidth]{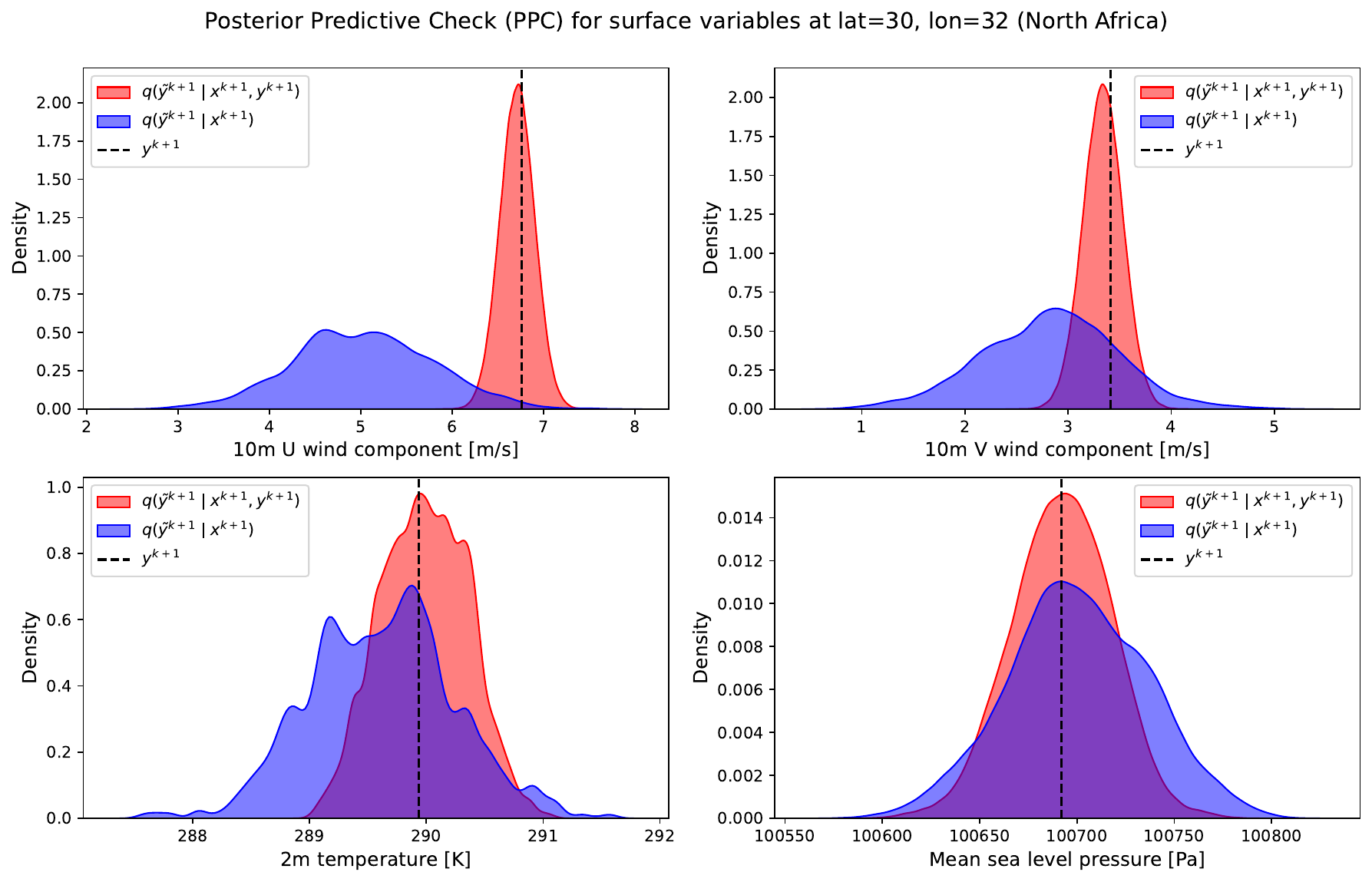}}
        \caption{
          Comparison between the distributions of conditional samples (red curve, generated using the optimal proposal) and unconditional samples (blue curve, generated with GenCast without conditioning) at an unobserved grid point (in North Africa). Observations (black dotted lines) are more likely in the distribution of conditional samples.
        }
        \label{fig:appendice_gencast_ppc_africa}
      \end{center}
    \end{figure}

\newpage
\section{Extension to stochastic interpolants} \label{appendix:stochastic interpolants}

\subsection{Stochastic interpolants for probabilistic forecasting} \label{Subsection:interpolants}
    Stochastic interpolants \cite{Interpolants} are a class of generative models that generalize diffusion \cite{diffusion} and flow matching \cite{FlowMatching}. They are designed to learn a transport between an easily sampled distribution $\rho_{0}$ and a target distribution $\rho_{1}$. In their linear form, commonly used in practice, they are defined by
    \begin{equation} \label{eq:linear_interpolant}
         x_{t} = \alpha_{t} x_{0} + \beta_{t}x_{1} + \gamma_{t}z, ~ t \in [0,1],
    \end{equation}
    where $(x_{0}, x_{1})$ is a data pair drawn from a joint measure $\nu(\mathrm{d}x_{0}, \mathrm{d}x_{1})$ with marginals $\rho_{0}(\mathrm{d}x_{0})$ and $\rho_{1}(\mathrm{d}x_{1})$,  $z \sim \mathcal{N}(0,I)$ with $(x_{0}, x_{1}) \perp z$, and $\alpha, \beta, \gamma $ continuous functions on $[0,1]$ that satisfy the following boundary conditions
    \begin{equation} \label{eq:boundary_conditions}
         \alpha_{0} = \beta_{1} = 1; \alpha_{1} = \beta_{0} = 0; \gamma_{0} = \gamma_{1} = 0.
    \end{equation}
    
     As shown by \citeauthor{probabilistic_forecasting_interpolants}, stochastic interpolants can be adapted for probabilistic forecasting. Given a dataset of successive states $\{(x^{k}, x^{k+1})_{i}\}_{i \in \mathcal{I}}$ from a dynamical system, one may set $x_1 = x^{k+1}$ and train a neural network $b_{\theta}$ to learn the velocity given $x^{k}$. The network is trained by minimizing
    \begin{equation} \label{eq:loss_interpolant}
         \mathcal{L}(\theta) = \int_{0}^{1}  \mathbb{E} \left [ \left \lVert \dot{x}_{t} - b_{\theta}(t, x_{t}, x^{k}) \right \rVert_{2}^{2} \right ] \mathrm{d}t,
    \end{equation}
    whose theoretical minimizer is $b_{t}(x_{t}, x^{k}) = \mathbb{E}[\dot{x}_{t} \mid x_{t}, x^{k}]$. Sampling from the transition law $p(x^{k+1} \mid x^{k})$ is then performed by solving the forward generative equation
    \begin{align} \label{eq:prior_forward_drift}
        &dx_{t} = b_{t}^{F}(x_{t}, x^{k})dt + \sqrt{2\varepsilon(t)}dw_{t}, \\
        &b_{t}^{F}(x_{t}, x^{k}) = b_{t}(x_{t}, x^{k}) + \varepsilon(t) s^{x}_{t}(x_{t}, x^{k}),
    \end{align}
    from $x_{0} \sim \rho_{0}$, where $s^{x}_{t}(x_{t}, x^{k})$ denotes the score of the conditional density of $x_{t}$ given $x^{k}$, and $\varepsilon(t) \geq 0$ is a diffusion coefficient. The score $s^{x}_{t}(x_{t}, x^{k})$ is not known a priori but can be expressed in terms of the velocity. For instance, when $x_{0} = x^{k}$ the relation between score and velocity reads
    \begin{equation} \label{eq:prior_score}
         s^{x}_{t}(x_{t}, x^{k}) = \frac{(\alpha_{t} \dot{\beta}_{t} - \dot{\alpha}_{t}\beta_{t})x^{k} - \dot{\beta}_{t}x_{t} + \beta_{t}b_{t}(x_{t}, x^{k})}{\gamma_{t} (\dot{\beta}_{t} \gamma_{t} - \beta_{t} \dot{\gamma}_{t})}.
    \end{equation}
    Using Equations \eqref{eq:prior_forward_drift} and \eqref{eq:prior_score}, we can then generate probable future states $x^{k+1}$ from a current state $x^{k}$, and thus emulate the system dynamics in an autoregressive manner.

\subsection{Sampling from the optimal proposal with stochastic interpolants}
    To apply the FA-APF, we must sample from the optimal proposal distribution $p(x^{k+1} \mid x^{k}, y^{k+1})$, which is generally infeasible for standard simulators. However, as shown by \citeauthor{FlowDAS} and \citeauthor{DAISI}, this distribution can be accessed using stochastic interpolants. The key idea is to incorporate the observation $y^{k+1}$ into the generative dynamics defined in Equation~\eqref{eq:prior_forward_drift}. This leads to the following posterior forward equation
    \begin{equation}\label{eq:posterior_forward_drift}
        dx_{t} = b_{t}^{F}(x_{t},x^{k}, y^{k+1})\mathrm{d}t + \sqrt{2 \varepsilon(t)}\mathrm{d}w_{t}.
    \end{equation} 
    The forward drift $b_{t}^{F}(x_{t},x^{k}, y^{k+1})$ is defined as the sum of a conditional velocity and a posterior score term
    \begin{equation} \label{eq:forward_drift_observation}
        b_{t}(x_{t}, x^{k}, y^{k+1}) + \varepsilon(t) s^{x,y}_{t}(x_{t}, x^{k}, y^{k+1}),
    \end{equation}
    where $s^{x,y}_{t}(x_{t}, x^{k}, y^{k+1})$ is the score of the density of $x_{t}$ given $x^{k}$ and $y^{k+1}$. The conditional velocity is not known a priori but, once again, can be derived directly from the velocity
    \begin{equation} \label{eq: posterior_drift}
        b_{t}(x_{t}, x^{k}, y^{k+1}) = b_{t}(x_{t}, x^{k}) + \lambda_{t} s^{y}_{t}(x_{t}, x^{k}, y^{k+1}),
    \end{equation}
    where $s^{y}_{t}(x_{t}, x^{k}, y^{k+1})$ is the score of the likelihood and $\lambda_{t}$ is a time-dependent coefficient given by
    \begin{equation}
        \lambda_{t} = \frac{\gamma_{t}(\dot{\beta}_{t}\gamma_{t} - \beta_{t}\dot{\gamma}_{t})}{\beta_{t}}.
    \end{equation}
    Thanks to Bayes' rule, the posterior score in Equation~\eqref{eq:forward_drift_observation} can be decomposed as
    \begin{equation} \label{eq:score_bayes}
        s^{x,y}_{t}(x_{t}, x^{k}, y^{k+1}) = s^{x}_{t}(x_{t}, x^{k}) + s^{y}_{t}(x_{t}, x^{k}, y^{k+1}).
    \end{equation}
    Since the prior score is already available from the velocity through Equation~\eqref{eq:prior_score}, the only unknown quantity that remains to be computed is the likelihood score $s^{y}_{t}(x_{t}, x^{k}, y^{k+1})$. To do so, we can use MMPS, the method introduced in Section \ref{subsection: Optimal_proposal}. 
    
    Putting all these elements together, the forward drift conditioned on an observation can be fully computed from the learned interpolant without additional training. The resulting procedure is summarized in Algorithm~\ref{algo:b_f}.

    \begin{algorithm}[h!]
        \caption{Computation of $b^{F}_{t}(x_{t}, x^{k}, y^{k+1})$}
        \begin{algorithmic}
            \STATE {\bfseries Inputs:} $t$, ~$x_{t}$, ~$x^{k}$, ~$b_{\theta}$, ~$y^{k+1}$, ~$\varepsilon(\cdot)$
    
                \vspace{0.2cm}
            
                \STATE \quad $b_{x} \gets b_{\theta}(t, x_{t}, x^{k})$
    
                \vspace{0.2cm}
                
                \STATE \quad $\displaystyle s_{x} \gets \frac{(\alpha_{t} \dot{\beta}_{t} - \dot{\alpha}_{t}\beta_{t})x^{k} - \dot{\beta}_{t}x_{t} + \beta_{t}b_{x}}{\gamma_{t} (\dot{\beta}_{t} \gamma_{t} - \beta_{t} \dot{\gamma}_{t})}$~(Eq.~\ref{eq:prior_score})
    
                \vspace{0.2cm}
            
                \STATE \quad $s_{y} \gets \mathrm{MMPS}(t, x_{t}, x^{k}, y^{k+1})$~(Eq.~\ref{eq:score_likelihood})
    
                \vspace{0.2cm}
                
                \STATE \quad  $\displaystyle b_{x,y} \gets b_{x} + \frac{\gamma_{t}(\dot{\beta}_{t} \gamma_{t} - \beta_{t} \dot{\gamma}_{t})}{\beta_{t}} s_{y}$~(Eq.~\ref{eq: posterior_drift})
    
                \vspace{0.2cm}
                
                \STATE \quad $s_{x,y} \gets s_{x} + s_{y}$~(Eq.~\ref{eq:score_bayes})
    
                \vspace{0.2cm}
                
                \STATE \quad $b^{F} = b_{x,y} + \varepsilon(t) s_{x,y}$~(Eq.~\ref{eq:forward_drift_observation})
    
                \vspace{0.2cm}
                
            \STATE {\bfseries Return} $b^{F}$ 
        \end{algorithmic}
        \label{algo:b_f}
    \end{algorithm}
    
\subsection{Computing weights with stochastic interpolants}
    The other important ingredient required to apply the FA-APF is the computation of particle weights (lines 5–10 of Algorithm \ref{algo:FA-APF}). To do so, we can use the same approximation as in Section \ref{subsection: weights}. It requires calculating $\mathbb{E}[x^{k+1} \mid x^{k}_{i}]$, which can be done directly using the learned velocity 
    \begin{equation} \label{eq:next_state_expectation}
         \mathbb{E}[x^{k+1} \mid x^{k}_{i}] = \frac{b_{0}(x^{k}_{i}, x^{k}_{i}) - \dot{\alpha}_{0}x^{k}_{i} }{\dot{\beta}_{0}},
    \end{equation}
    under the assumption that $\dot{\beta}_{0} \neq 0$.

\newpage
\section{Additional details on Algorithm \ref{algo:FA-APF}} \label{appendix:details_algo}
\subsection{Computational complexity}
Using the notation of Algorithm \ref{algo:FA-APF}, let $K$ denote the number of filtering steps, $N$ the number of particles, $T$ the number of diffusion steps required to solve Eq.~\eqref{eq:reverse_diffusion_equation}, and $C_{\text{step}}$ the cost of a single diffusion step. Since the computational cost is dominated by posterior sampling (line 13 of Algorithm \ref{algo:FA-APF}), the overall complexity of the algorithm is $\mathcal{O}\left( K \times N \times T \times C_{\text{step}} \right)$. However, because posterior sampling is fully parallelizable across particles, the effective complexity can be reduced by a factor $N$, and we therefore omit the explicit loop over particles in Algorithm \ref{algo:FA-APF} for readability.

As explained in Section $\ref{section:Method}$, each diffusion step is based on MMPS \cite{MMPS} and therefore involves solving a linear system (see Eq.~\ref{eq:score_likelihood}). In practice, this system is solved iteratively using GMRES \cite{GMRES}, which involves vector–Jacobian products and leads to a cost of approximately $C_{\text{step}} = \mathcal{O}(M \times C_{\text{denoiser}})$, where $M$ is the number of GMRES iterations and $C_{\text{denoiser}}$ the cost of a pass through the denoiser. Thus, if posterior sampling can be parallelized across particles, the computational cost of the proposed method is $\mathcal{O}(K \times T \times M \times C_{\text{denoiser}})$. Importantly, FA-APF is compatible with any posterior sampling methods, including computationally cheaper alternatives such as DPS \cite{DPS}, although we adopt MMPS here due to its superior empirical performance. 

Regarding memory complexity, the main bottleneck also comes from solving the linear system at each diffusion step when sampling from the optimal proposal. Indeed, since the covariance matrix $\mathbb{V}[x^{k+1} \mid x_{t}^{k+1}, x^{k}]$ is too large to be stored explicitly in high-dimensional systems, we instead solve a linear system using implicit access to $\mathbb{V}[x^{k+1} \mid x_{t}^{k+1}, x^{k}]$ through the second-order Tweedie's formula (see Eq.~\ref{eq:Tweedie_2nd_ordre}). This requires the Jacobian of the denoiser, which is obtained via automatic differentiation at the cost of approximately storing the denoiser activations. For large models such as GenCast (Section \ref{subsection: GenCast}), this can be particularly demanding in terms of VRAM. However, we were able to run the 1° resolution denoiser on H100 GPUs with 80 Gb of VRAM without relying on memory-saving techniques such as gradient checkpointing.

\subsection{Approximations}
Several approximations are introduced in Algorithm \ref{algo:FA-APF} to make FA-APF tractable in practice. First, the posterior sampling method (MMPS, \citeauthor{MMPS}) used to sample from the optimal proposal is not exact. Indeed, the local diffusion distribution $p(x^{k+1} \mid x^{k+1}_{t}, x^{k})$ is approximated by a Gaussian distribution whose moments are estimated using Tweedie's formulas (see Appendix \ref{appendix:Tweedie}). While this approximation becomes accurate at the end of the reverse diffusion process (i.e. for low noise levels), it is less accurate at earlier stages where $p(x^{k+1} \mid x^{k+1}_{t}, x^{k})$ is typically multimodal and far from Gaussian. An interesting direction for future work would be to rigorously evaluate different posterior sampling methods \cite{DPS, MMPS, EnKG} on weather data using coverage and accuracy metrics \cite{TARP, MIRA}.

As explained in Section \ref{section:Method}, particle weights are also inexact, as they are obtained by approximating $p(x^{k+1} \mid x^{k})$ with a Dirac centered at $\mathbb{E}\left[ x^{k+1} \mid x^{k} \right]$. In preliminary experiments on Lorenz63, we found that approximate and exact weights (computed using Monte Carlo) lead to similar results when the number of particles is sufficiently large. However, in high-dimensional settings, the number of particles is typically small compared to the dimension of the system, making improved weight approximations an important avenue for future research.

Finally, although FA-APF uses the optimal proposal to minimize the variance of particle weights, inflation of the observation covariance matrix remains necessary in high-dimensional systems to prevent weight collapse. Although this biases the approximation of the filtering distribution by modifying the likelihood, it leads to strong empirical results while preserving particle diversity.


\end{document}